%% file: main.tex
\documentclass[runningheads,orivec]{llncs}
\usepackage[T1]{fontenc}
\usepackage{graphicx}
\input{preamble}

\usepackage{hyperref}
\usepackage{enumitem}
\usepackage{color}

\begin{document}
\title{Hybrid Robustness Verification for Spatio-Temporal Neural Networks}

\author{\CR{Sherwin Varghese\inst{1}\orcidID{0000-0003-0191-8107} \and
Matthew Wicker\inst{1}\orcidID{0000-0003-0779-3114} \and
Alessio Lomuscio\inst{1}\orcidID{0000-0003-3420-723X}}
\thanks{Accepted at the 9th International Symposium on AI Verification (SAIV 2026).}
}

\authorrunning{S. Varghese et al.}

\institute{\CR{Imperial College London, United Kingdom} \\
\email{\CR{\{sherwin.varghese, m.wicker, a.lomuscio\}@imperial.ac.uk}}
}

\maketitle              

\input{sections/0_abstract}
\input{sections/1_introduction}

\input{sections/2_related_work}
\input{sections/3_preliminaries}
\input{sections/4_method}
\input{sections/5_experiments}

\input{sections/6_conclusion}
\input{sections/7_acknowledgements}
\input{sections/9_appendix}

%
%
%
\clearpage
\bibliographystyle{splncs04}
\bibliography{main}

\end{document}

%% file: preamble.tex
\usepackage{amsthm}
\usepackage{amsmath, amssymb, graphicx}
\usepackage{accents}
\usepackage{multirow}
\usepackage{xcolor}
\usepackage{booktabs}
\usepackage{multirow}
\usepackage{tabularx}
\usepackage{algorithm}
\usepackage{algorithmic}
\usepackage{rotating}

\newcommand{\norm}[1]{\left\lVert#1\right\rVert}
\newcommand{\ubar}[1]{\underaccent{\bar}{#1}}

\newcommand{\CR}[1]{#1}

\makeatletter
\renewcommand\section{\@startsection{section}{1}{\z@}%
  {-1.2ex \@plus -0.4ex \@minus -0.2ex}%
  {0.6ex \@plus 0.2ex}%
  {\normalfont\large\bfseries\boldmath
   \rightskip=\z@ \@plus 1fil\relax}}
\renewcommand\subsection{\@startsection{subsection}{2}{\z@}%
  {-1.0ex \@plus -0.3ex \@minus -0.2ex}%
  {0.4ex \@plus 0.2ex}%
  {\normalfont\normalsize\bfseries\boldmath
   \rightskip=\z@ \@plus 1fil\relax}}
\renewcommand\subsubsection{\@startsection{subsubsection}{3}{\z@}%
  {-0.8ex \@plus -0.3ex \@minus -0.2ex}%
  {0.3ex \@plus 0.2ex}%
  {\normalfont\normalsize\bfseries\boldmath
   \rightskip=\z@ \@plus 1fil\relax}}
\makeatother

\everypar{\looseness=-1 }

%% file: sections/0_abstract.tex
\begin{abstract}

\looseness=-1
With AI increasingly deployed in safety-critical systems, providing formal robustness guarantees for the underlying models is essential. Existing verification methods either rely on overly conservative approximations or incur prohibitive computational costs. For example, the use of $\ell_p$-norm perturbations in video settings encodes the belief that the adversary can inject noise in every video frame. In practice, adversarial perturbations exhibit structured spatial and temporal correlations, constrained to lower-dimensional, semantically meaningful subspaces. In this work, we study robustness verification of \emph{3D convolutional neural networks} (3D CNNs) processing video and volumetric inputs, targeting applications in action recognition (UCF-101), autonomous driving (Udacity), and medical imaging (MedMNIST) exploiting realistic assumptions on adversarial strength by modelling them as spatio-temporal constraints --- where the attacker can modify either a subset of frames or patches within a set of consecutive frames. We demonstrate that modelling realistic constraints enables tighter approximations, particularly in CNNs. We introduce \emph{Spatio-Temporal Bound Propagation} (STBP), a verification framework that computes an exact closed-form characterization of the first convolutional layer and propagates certified bounds through subsequent layers using scalable approximations. Computing the exact closed-form provides the tightest bounds for the first convolutional layer. Thus, we utilise approximation methods in the remainder of the network. To spur further progress in this field, we propose \texttt{ST-Bench}, a verification benchmark for autonomous driving and activity recognition, to systematically evaluate verifiable robustness. Compared to existing verification and training-based approaches, STBP provides stronger robustness guarantees with significantly improved scalability, achieving up to $1.7\times$ higher certified robust accuracy under identical perturbation budgets.

\end{abstract}

%% file: sections/1_introduction.tex
\section{Introduction}\label{sec:introduction}

\looseness=-1
Deep learning models achieve state-of-the-art performance in safety-critical domains such as autonomous driving and medical imaging, but remain vulnerable to adversarial perturbations—small, often imperceptible input changes that induce incorrect predictions~\cite{szegedy2013intriguing,goodfellow2015explaining,fawzi2018analysis}. Such failures pose serious risks in high-stakes applications, motivating the use of formal verification to provide rigorous robustness guarantees. Existing verification methods typically cast robustness as a non-convex optimization problem and rely on convex relaxations, including SMT solving~\cite{katz2017reluplex,katz2019marabou}, mixed-integer programming~\cite{tjeng2018evaluating}, abstract interpretation~\cite{gehr2018ai2,singh2019abstract}, and relaxation-based techniques~\cite{wong2018provable,dvijotham2018dual}. Yet, practical adoption of certification methods have lagged behind the wide-spread adoption of machine learning models in practice due to methods being (A) overly-approximate and therefore unable to provide actionable guarantees or (B) too computationally expensive to provide guarantees at scale. Both of these issues are substantially exacerbated with increasing feature dimensionality where methods from abstract interpretation (i.e., interval bound propagation) are far too approximate to provide actionable guarantees~\cite{gowal2018ibp} and methods that are exact in the limit (i.e., mixed integer linear programming) are too computationally expensive to be practically useful~\cite{lin2020last,lee2020mixed,lambert2024benchmark}. This limits current verification to applications with low-dimensional datasets that can be solved by small to medium models~\cite{katz2019marabou,banerjee2024relational}.

\begin{figure*}[t]
	\centering
	\includegraphics[alt={Architecture diagram of Spatio-Temporal Bound Propagation: input video tensor is fed through structured perturbation encoding, an exact closed-form first convolutional layer, and three alternative bound-propagation pipelines (IBP, Lipschitz, Löwner--John) for the remaining layers.}, width=0.9\linewidth]{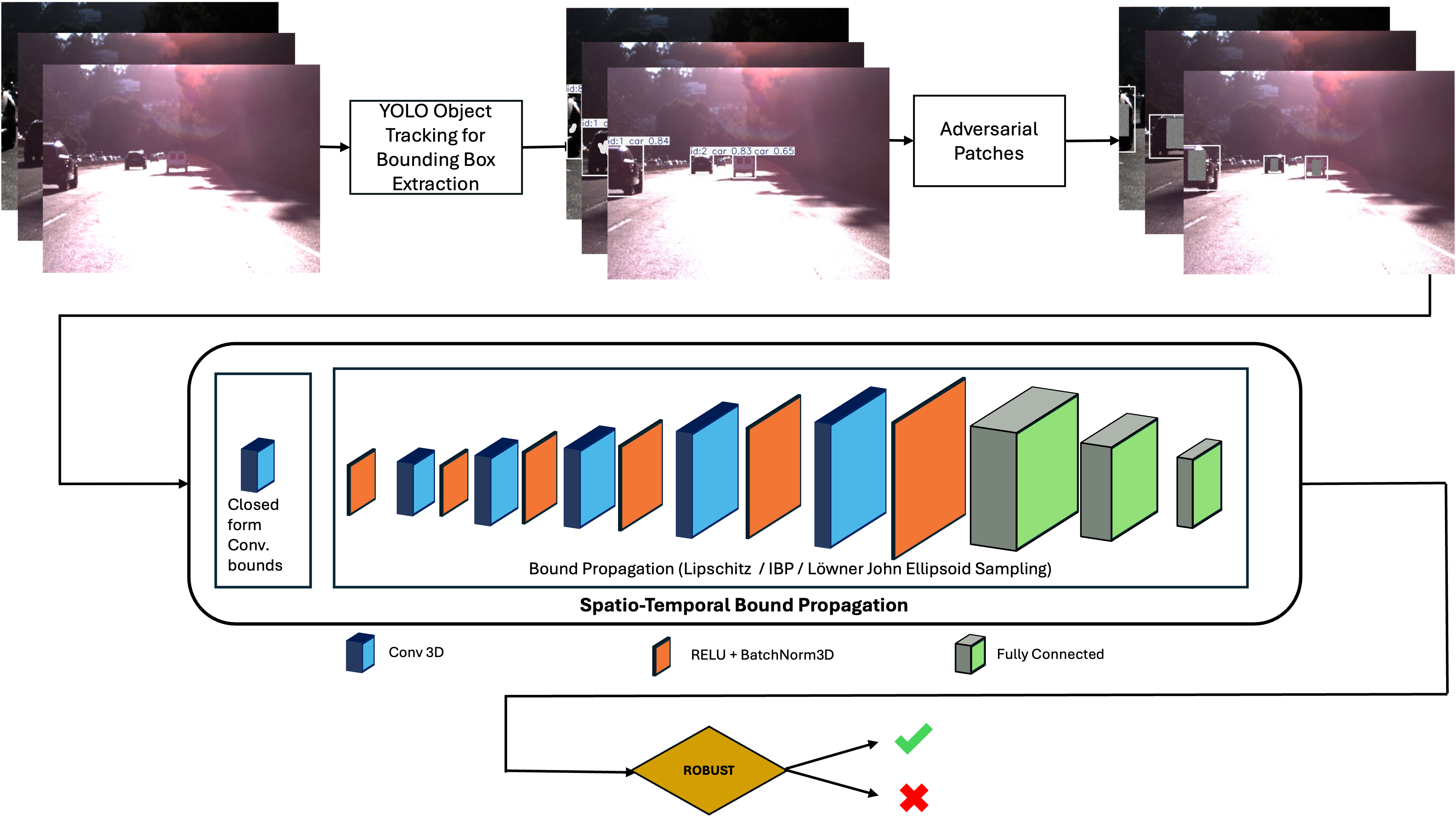}
	\centering
	\caption{\CR{Overview of Spatio-Temporal Bound Propagation (STBP) for verifying spatio-temporal adversarial patch propagation.}}
	\label{fig:stbp-arch}
\end{figure*}

To adopt verification for high dimensional inputs, we introduce new computational approaches that impose realistic constraints, exploiting prior knowledge about spatio-temporal correlations between feature dimensions. We argue that a key source of their failure in high-dimensional settings is the choice of adversarial model. Standard $\ell_p$-norm perturbations treat every input dimension as independently attackable, which is appropriate for static images but fundamentally misrepresents the structure of real-world threats in video~\cite{wang2024benchmarking,luo2018fast} and volumetric medical imaging~\cite{wu2024efficient,che2023towards,che2023image}. In these domains, physically plausible adversaries act by manipulating objects in the scene --- such as placing an adversarial road sign or bumper sticker visible across consecutive frames --- or through sensor noise that is inherently spatially correlated~\cite{tramer2019adversarial,agarwal2020noise,hingun2023reap}. Both threat types induce perturbations that are shared or correlated across the temporal and spatial axes of the input. Adopting this \emph{spatio-temporal attack model} directly reduces the effective adversarial dimension: rather than an independent perturbation per pixel per frame, the attacker controls a small set of shared degrees of freedom. To exploit this structure, we introduce \emph{Spatio-Temporal Bound Propagation} (STBP), which encodes the spatio-temporal constraints as an affine transformation and derives an exact closed-form characterisation of the reachable set at the initial convolutional layers. The resulting bounds are then propagated through the remainder of the network via abstract interpretation, preserving soundness while remaining computationally feasible.

\looseness=-1
We experimentally validate our approach on a variety of datasets including the UCF-101 dataset~\cite{soomro2012ucf101}, the Udacity self-driving car dataset~\cite{udacity2017udacity}, and various medical imaging benchmarks \cite{yang2021medmnistv2}. To understand spatio-temporal robustness in these settings, we generate novel, sophisticated verification benchmarks. In autonomous driving, we have curated a dataset simulating potential regions of each frame where one might introduce adversarial road signs or bumper stickers in order to understand the worst-case robustness to such introductions. In order to spur further growth in the investigation of practically robust neural networks on video and volumetric medical imaging data, we additionally release our curated benchmarking datasets accessible via the following~\href{https://zenodo.org/records/15426146}{link}~\footnote{https://zenodo.org/records/15426146}.

In summary, we make the following contributions: 
\begin{itemize}[
    leftmargin=*,
    itemsep=0pt,
    parsep=0pt,
    topsep=2pt,
    partopsep=0pt
]
    \item  \CR{Identify the regimes in which existing formal verification techniques become inefficient or overly conservative for spatio-temporal systems; due to explosive bounds or compute-time constraints and motivate the notion of realistic adversarial perturbations.}
	\item Introduce spatio-temporal bound propagation (STBP) that efficiently verifies practical robustness properties in high-dimensional domains by exploiting known spatio-temporal feature correlations. 
	\item Establish that STBP substantially out-performs existing robust verification with IBP, by over 1.7 times on our real-world activity recognition and autonomous driving safety specifications.
	\item Curate and publicly release a set of real-world benchmark dataset for safety-critical domains such as video classification and autonomous driving with novel adversarial specifications. 
\end{itemize}
The rest of the paper is organised as follows. 

In section~\ref{sec:method}, we describe the proposed method detailing realistic adversarial constraints, spatio-temporal bound propagation and spatio-temporal adversarial attacks. We demonstrate the effectiveness of the approach in section~\ref{sec:evaluation} on a variety of datasets including our curated ST-Bench verification benchmark.

%% file: sections/2_related_work.tex
\section{Related Work}
\label{sec:related-work}

Formal verification attempts to provide provable guarantees to ensure that a certain property of the system holds. In this particular case, the property to be verified is robustness. Early works formulate the robustness verification as constraint optimisation problem using mixed integer programming (MIP)~\cite{lomuscio2017anapproach,dutta2019reachability,tjeng2019evaluating,botoeva2020efficient}. These works provide soundness guarantees, however, they struggle to verify larger neural networks \cite{sosnin2024certified,sosnin2025abstract}. In addition, these works use MIP to model the entire network composed of several layers as a set of constraints. However, our work focuses on tightening the initial convolution layer of the network perturbations that vary across the temporal dimension to obtain the initial constraints. Further layers of the network rely on the less expensive IBP to aid faster verification.

On the one hand, ~\cite{jordan2O22zonotope,singh2019abstract,balunovic2019certifying} model the networks using polyhedral constraints (surveyed in \cite{huchette2026deep}) and certify robustness for image perturbations, they employ abstract verification techniques that result in loose approximations, especially for higher dimensional inputs, whereas on the other hand, ~\cite{wang2023towards} tackles the problem from a different angle by performing geometric transformations iteratively on the inputs and analysing the drop in model accuracy to determine the global worst case transformation. Since it is impossible to practically cover all possible geometric transformations, the soundness of this approach cannot be guaranteed. \CR{Other sound bound-propagation-based methods explore certification of finite-horizon robustness properties \cite{wicker2021certification,lechner2021infinite,wicker2024probabilistic} or using mixed propagation strategies for forward and backwards passes \cite{sosnin2026exact,wicker2025certification}. While these methods mix propagation strategies and incorporate temporal certification they do not consider the shared constraints that we impose in this work.}

There are not many works that explore robustness verification of video classifiers and spatio-temporal systems, given their limited ability to scale for large scale inputs. While~\cite{mummadi2019defending} proposes shared adversarial perturbations, the approach focuses on adversarial defence. \cite{hu2023robustness} proposes robustness verification, however, this work focuses on motion camera perturbations on point cloud image frames. \cite{lee2023defending} proposes an adversarial defence model for activity recognition models based on PGD. \CR{Prior works including \cite{wu2020robustness} provide robustness guarantees for video classifiers via local Lipschitz-based reasoning (based on \cite{wicker2018feature}) together with a Monte Carlo–style search over adjacent frames, certifying small networks under per-frame $\ell_\infty$ perturbations. Related complementary results on the formal foundations for video-domain verification appear in~\cite{wu2020gamebased}. In contrast, our STBP framework exploits the closed-form structure of shared spatio-temporal perturbations to obtain exact first-layer bounds and combines them with sound IBP/Lipschitz propagation for deeper layers. More recently, \cite{zhang2025h2v} scaled robustness validation to large networks (including ResNets, vision transformers, and 3D CNNs for video classification) via H\"older optimisation over a Hilbert space-filling reformulation, but addresses a different threat model -- geometric perturbations rather than the shared, frame-correlated $\ell_\infty$ perturbations targeted by STBP.} This work proposes robustness verification of spatio-temporal systems, including video recognition models.

Recent advances in improving robustness incorporate certified training, predominantly employing IBP for verification~\cite{de2024expressive,mao2023connecting,muller2023certified}. While these works do not consider spatio-temporal systems, we propose spatio-temporal attacks that employ shared constraints. While the attack strength may be lower than the attacks proposed in these methods, this allows modelling realistic adversaries for spatio-temporal systems and enables robustness verification to scale to higher dimensional inputs. 

In addition to this, the prior work~\cite{chiang2020certified} explores adversarial patches limited to images in relatively smaller datasets. By incorporating adversarial patches with shared adversarial perturbations better scalability to spatio-temporal systems with real-world $l_\infty$ perturbations is achieved. VideoStar~\cite{sasaki2025robustness} introduced as an extension of the ImageStar~\cite{tran2020cav} for formal verification of video classification neural networks. VideoStar represents reachable sets as tuples $(c, V, P, l, u)$, where $c$ is an anchor video, $V$ is a set of generator videos, and $P$ encodes linear constraints on the combination coefficients. 

\looseness=-1
Lipschitz continuity has long been used to reason about robustness and verification of neural networks by relating bounded sensitivity to certified guarantees under norm-bounded perturbations~\cite{szegedy2014intriguing,hein2017formal}. Prior work focused on efficiently estimating or bounding Lipschitz constants of deep models, including spectral norm–based bounds for feed-forward and convolutional networks~\cite{virmaux2018lipschitz} and convex or semidefinite relaxations for tighter estimates~\cite{fazlyab2019efficient}. These bounds have been leveraged to derive robustness margins and verification guarantees, for instance through Lipschitz-margin training~\cite{tsuzuku2018lipschitz} or constrained architectures such as skew-orthogonal convolutions~\cite{singla2021skew}. More recently, layer-wise Lipschitz propagation has enabled scalable certification for large models and structured perturbations~\cite{massena2025fastrobcerts}. While such methods scale well, they often yield conservative bounds due to uniform worst-case sensitivity. In contrast, our approach combines an exact closed-form characterization of the first spatio-temporal convolutional layer with Lipschitz-based propagation in deeper layers, reducing conservatism while preserving scalability.

%% file: sections/3_preliminaries.tex
\section{Preliminaries}
\label{sec:preliminaries}

\looseness=-1
We consider a Feed-Forward Neural Networks (FFNN) that perform classification of temporal sequences. In particular, our focus is on video classification systems that takes in a set of input video frames as input. Formally, we define the network $f_{\theta}$ with parameters $\theta$ as $f_{\theta}: \CR{\mathbb{R}^{C \times D \times H \times W}} \rightarrow \mathbb{R}^{M}$ where \CR{$C$ denotes the number of channels (for RGB, $C{=}3$), $D$ the temporal depth (number of frames per video), $H$ the height, and $W$ the width of each input frame}; $M$ denotes the number of output classes. The network assigns a class $i \in \{1, 2, \dots, M\}$ if $f_{\theta}(x)_i > f_{\theta}(x)_j\ \forall i \neq j$.

\looseness=-1
Given the FFNN $f_{\theta}$ and the input video (represented by several image frames)  $X$, the clean prediction of the network is: $\hat{y} = f_{\theta}(x)$.
The FFNN is composed of intermediate layers. We denote $x_0$ as the input layer. The final output of the network is denoted by $y_\textsubscript{target}$. The neural network is composed of a series of transformations $h_l$ within each of its $L$ layers. 
The output of each layer $z_l$ can be expressed as $z_l = h_l(x_\textsubscript{$l-1$})$ for $l = 1, 2, ... L$. The transformations $h_l$ can be an affine transformation or a monotonic function such as ReLU or sigmoid. We represent the affine transformation as $h_l(x_\textsubscript{$l-1$}) = W_l\ \cdot  x_{l-1} + b_l$, where $W_l \in \CR{\mathbb{R}^{C \times D \times H \times W}}$ is the weight of the matrix and $b_l$ is the bias vector for the layer $l$.

\subsection{Robustness Verification}
\looseness=-1
Verification techniques validate whether neural networks satisfy certain specifications. In this setting, we focus on ensuring that the adversarial robustness specification holds.

\looseness=-1
In order to verify the robustness of $f_{\theta}$ against adversarial inputs, we apply perturbations on the input video $X \in \CR{\mathbb{R}^{C \times D \times H \times W}}$, where \CR{$C$ denotes the number of channels and $D$ the temporal depth (number of frames)}. We observe the outputs of $f_{\theta}$ under the adversarial perturbation defined by the budget $\epsilon$. Thus, verification proceeds by searching for a counter example that violates the specification or until all the samples satisfy the specification constraint.
As defined by \cite{gowal2018ibp}, the network is called robust at a point $x_0$ if no adversarial example can cause the prediction of the network to change from its target label for all $z_0 \in \chi(x_0)$, where $\chi(x_0)$ is a set of inputs around $x_0$. This is expressed formally for each class $y$ as $(e_{\hat{y}'} - e_{\hat{y}})^T z_l \leq 0 \quad \forall z_0 \in \chi(x_0) = \{ x | \norm{x - x_0}_\infty < \epsilon \}$.

\looseness=-1
We apply a perturbation $\delta$ bounded by $\epsilon$ to the input video frames $X$ within a norm $\ell_\infty $-norm or $ \ell_2 $-norm $\epsilon \geq 0$, \ $\epsilon$ being the maximum allowed perturbations of each frame in the video.

For the adversarial input $X^\text{adv} = X + \delta$, the output of the network~\cite{goodfellow2014explaining} is defined by $\hat{y}' = f_{\theta}(X^\text{adv}).$

For a robust network, the input with perturbations $X +\delta$ maintains the same prediction $\hat{y}' = \hat{y}$

\subsection{Spatio-Temporal Systems}
\looseness=-1
Spatio-temporal systems refer to neural network architectures that process inputs exhibiting inherent correlations along both spatial and temporal dimensions. Formally, such systems implement mappings $f : \mathbb{R}^{C \times D \times H \times W} \rightarrow \mathcal{Y}$, where $C$ denotes the number of input channels, $D$ the temporal depth (e.g., number of frames or slices), and $H \times W$ the spatial resolution. Prominent examples include video classification networks processing sequential frames captured over time, and volumetric medical imaging models analyzing sequential slices from magnetic resonance imaging (MRI) or computed tomography (CT) scans. A distinguishing characteristic of spatio-temporal data is that realistic perturbations - whether arising from sensor noise, environmental artifacts, or adversarial manipulation - are not independent across the temporal axis. For instance, an adversarial patch affixed to a physical object (e.g., a manipulated road sign or bumper sticker) induces perturbations that are spatially localized and temporally correlated across consecutive frames. This structure reduces the effective adversarial degrees of freedom compared to unconstrained per-pixel perturbations, a property that can be exploited to obtain tighter verification bounds than those achievable with standard interval bound propagation methods that treat all input dimensions as independent.

\subsection{Challenge with Standard IBP in Spatio-Temporal Settings}

\looseness=-1
Interval Bound Propagation~\cite{gowal2018ibp} becomes inefficient when applied to inputs with high correlation, such as videos. In video classification, the input is a sequence of  \CR{$D$} frames. Due to high correlation along the temporal axis, the independent propagation of bounds for each frame leads to an \textit{exponential} or \textit{polynomial} increase in the size of the propagated bounds.

\looseness=-1
For example, let \CR{$D$} be the number of frames in a video, and suppose the perturbation is applied independently to each frame. The propagated bounds across all frames grow as \CR{$O(D^d)$}, where $d$ is the input dimension. This results in bounds that are overly conservative and difficult to scale for large networks. To address this, we can introduce shared constraints across the temporal axis, allowing for more efficient certification in video models.

\subsection{Zonotope Abstraction in Spatio-Temporal Settings}
\looseness=-1
Zonotope abstractions~\cite{singh2018fast} offer a principled alternative to IBP by representing the set of reachable activations as a zonotope $Z = \{c + \sum_{i=1}^{m} \alpha_i g_i \mid \alpha_i \in [-1,1]\}$, where $c$ is the center and $g_i$ are generators encoding directions of uncertainty. For $\ell_\infty$ perturbations on input $x$, the perturbed input is represented as $x' = x + \sum_{i=1}^{n} \varepsilon \cdot e_i \cdot \alpha_i$, where $e_i$ are unit vectors and each $\alpha_i$ is an independent noise symbol. Passing this through a linear layer $z = Wx + b$ yields $z = Wx + b + \sum_i \varepsilon \cdot W_{\cdot,i} \cdot \alpha_i$, which recovers exactly the IBP interval bound $z \in [Wx + b - \varepsilon |W| \mathbf{1},\ Wx + b + \varepsilon |W| \mathbf{1}]$. Whilst zonotopes can yield tighter bounds than IBP through the propagation of correlated noise symbols, we do not adopt this approach for two reasons.

\looseness=-1
\textbf{Closed-form bounds on the initial convolution are tighter.} Our method exploits the structure of spatio-temporal inputs by deriving \textit{closed-form} analytical bounds through the first convolutional layer. These bounds directly account for the spatial support and temporal redundancy of the perturbation set, yielding significantly tighter enclosures than what zonotope propagation achieves at this stage. Since the dominant source of looseness in bound propagation is the first few layers, replacing zonotope abstraction here with an exact closed-form expression provides a qualitatively better starting point for the subsequent propagation.

\looseness=-1
\textbf{Zonotope propagation 
is computationally prohibitive.} The number of generators in a zonotope grows with each layer as new affine dependencies are introduced, making the representation increasingly expensive to maintain. For video inputs of dimension \CR{$C \times D \times H \times W$}, the initial zonotope already contains \CR{$O(C \cdot D \cdot H \cdot W)$} generators, and this count is compounded through successive layers. In practice, propagating a zonotope through an entire 3D convolutional network incurs a runtime cost that substantially exceeds our proposed approaches. Specifically, our IBP-based and Lipschitz-based certification methods complete in a fraction of the time required by zonotope propagation, making them far more suitable for the spatio-temporal networks considered in this work. We note that our proposed L\"{o}wner-John (LJ) ellipsoid sampling method operates under a different computational regime, as it draws samples from the minimum-volume ellipsoid enclosing the certified input set rather than propagating a closed symbolic representation, and is therefore not directly compared on this dimension.

%% file: sections/4_method.tex
\section{Method}
\label{sec:method}

We consider a supervised learning setting in which a model \( f: \CR{\mathbb{R}^{C \times D \times H \times W}} \to \mathcal{Y} \) maps an input tensor \( \mathbf{x} \in \CR{\mathbb{R}^{C \times D \times H \times W}} \) to an output in a label space \( \mathcal{Y} \). Here, \( C \) denotes the number of channels, \CR{\( D \) the length of the temporal or sequential axis (e.g., video frames or imaging slices), and \( H,W \) the spatial height and width}. The model \( f \) is typically realized as a deep neural network composed of alternating linear transformations and non-linear activations. We assume a data distribution \( \mathcal{D} \) over input-label pairs \( (\mathbf{x}, y) \in \CR{\mathbb{R}^{C \times D \times H \times W}} \times \mathcal{Y} \), and we are interested in verifying the behaviour of \( f \) not only on individual inputs but over neighbourhoods defined by allowable perturbations.

To formalize this, we define an input specification \( T(\mathbf{x}) \subseteq \CR{\mathbb{R}^{C \times D \times H \times W}} \) which captures the set of inputs deemed semantically equivalent to a given input \( \mathbf{x} \). The goal of robustness verification is to prove that the model's prediction remains invariant over this set. Specifically, we seek to certify that \( f(\mathbf{x}') = f(\mathbf{x}) \) for all \( \mathbf{x}' \in T(\mathbf{x}) \). When such a certificate can be produced, it provides a formal guarantee that the model is robust to all perturbations allowed by the specification. However, in high-dimensional input spaces, such as those arising in spatio-temporal data, existing verification techniques often fail to scale due to the complexity of propagating bounds through deep models. This motivates the development of new techniques tailored to settings where the input exhibits spatial and temporal structure.

\subsection{Modelling Spatio-Temporal Constraints}
We now define a closed form formulation to compute tight, certified bounds on the activations of the first layer of a neural network under structured perturbations. Let \( \mathbf{x} \in \mathbb{R}^{C \times D \times H \times W} \) be a 4D input tensor, and let \( T(\mathbf{x}) \subseteq \mathbb{R}^{C \times D \times H \times W} \) be an input specification defining the set of valid perturbed inputs, constrained to reflect realistic spatio-temporal consistency.
Let \( \boldsymbol{\delta} \in \mathbb{R}^{n} \) be the flattened perturbation tensor, with \( n = CDHW \). The perturbed input is \( \mathbf{x} + \boldsymbol{\delta} \) denoted by \(\mathbf{x}'\). We impose the following structured constraints on \( \boldsymbol{\delta} \):
\begin{itemize}[
    leftmargin=*,
    itemsep=0pt,
    parsep=0pt,
    topsep=2pt,
    partopsep=0pt
]
	\item \textbf{Bounded perturbations.}
	      For \( i \in \mathcal{B} \subseteq \{1, \dots, n\} \), we impose: $-\epsilon_i \leq \delta_i \leq \epsilon_i,$ where $\epsilon_i > 0$ is the perturbation magnitude for the $i$-th input dimension.
	\item \textbf{Shared perturbations.}
	      For index pairs \( (i, j) \in \mathcal{S} \), we enforce: $\delta_j = \delta_i.$
	\item \textbf{Non-perturbations.}
	      For \( i \in \mathcal{F} \subseteq \{1, \dots, n\} \), we fix: $\delta_i = 0.$
\end{itemize}
These constraints reduce the adversarial degrees of freedom, capturing inductive biases such as temporal smoothness or anatomical consistency across frames. For an affine layer, the pre-activation output is given by $\mathbf{z} = \mathbf{W}(\mathbf{x} + \boldsymbol{\delta}) + \mathbf{b} \in \mathbb{R}^m$.

\subsection{Exact Closed-Form Bounds for the First Layer}

We first analyze the first affine layer \(f^{(1)}\). Let \(\mathbf{W} \in \mathbb{R}^{m \times n}\) and \(\mathbf{b} \in \mathbb{R}^{m}\) denote the effective weight matrix and bias, with \(n = CDHW\). The pre-activation output is $\mathbf{z}^{(1)}(\mathbf{x}') = \mathbf{W}(\mathbf{x} + \boldsymbol{\delta}) + \mathbf{b} = \mathbf{z}^{(1)}_{\mathrm{clean}} + \mathbf{W}\boldsymbol{\delta}$, where \(\mathbf{z}^{(1)}_{\mathrm{clean}} := \mathbf{W}\mathbf{x} + \mathbf{b}\) denotes the clean (unperturbed) pre-activation output of the first layer. For a neuron \(j \in \{1,\dots,m\}\), the perturbation-induced deviation is $\Delta z^{(1)}_j = \sum_{i=1}^{n} w_{j,i} \delta_i$.

\subsubsection{Reduction under Shared Perturbations}

The shared-perturbation constraints \((i,k) \in \mathcal{S}\) enforce \(\delta_i = \delta_k\), inducing equivalence classes over the perturbation indices. Let \(\tilde{\boldsymbol{\delta}} \in \mathbb{R}^{q}\) denote the vector of independent perturbation variables after enforcing these equalities. Then, the perturbation term can be rewritten as $\Delta z^{(1)}_j = \sum_{r=1}^{q} \alpha_{j,r} \tilde{\delta}_r$, where each coefficient \(\alpha_{j,r}\) aggregates all weights \(w_{j,i}\) associated with indices tied to \(\tilde{\delta}_r\). Each independent perturbation satisfies a box constraint $-\epsilon_r \le \tilde{\delta}_r \le \epsilon_r$, with \(\tilde{\delta}_r = 0\) for variables corresponding to \(\mathcal{F}\).

\subsection{Spatio-Temporal Bound Propagation}

Given a certified input specification \( T(\mathbf{x}) \subseteq \mathbb{R}^{n} \) defined by structured perturbation constraints, our goal is to compute sound bounds on the output of a neural network \( f \) under all admissible perturbations \( \mathbf{x}' \in T(\mathbf{x}) \). To achieve this, we introduce \textit{Spatio-Temporal Bound Propagation} (STBP), an exact, closed-form analysis of the first layer with scalable, sound relaxation-based bound propagation for all subsequent layers. This hybrid strategy preserves tightness while ensuring computational tractability for deep architectures. Algorithm~\ref{alg:stbp-overview} provides a high-level overview of the full STBP pipeline.

\begin{algorithm}[H]
    \caption{STBP: Spatio-Temporal Bound Propagation (Overview)}
    \label{alg:stbp-overview}
    \begin{algorithmic}[1]
        \REQUIRE Input $\mathbf{x}$, network $f = f^{(k)} \circ \cdots \circ f^{(1)}$, specification $T(\mathbf{x})$
        \STATE Encode spatio-temporal constraints: shared $\mathcal{S}$, bounded $\mathcal{B}$, fixed $\mathcal{F}$
        \STATE Reduce perturbation to independent variables $\tilde{\boldsymbol{\delta}} \in \mathbb{R}^q$
        \STATE Compute exact first-layer bounds $[\ell^{(1)}, u^{(1)}]$ via closed-form (Theorem~1)
        \STATE Choose propagation method for layers $2, \dots, k$:
        \STATE \quad \textbf{Method A:} IBP (fast, less tight)
        \STATE \quad \textbf{Method B:} Lipschitz propagation (moderate cost, tighter)
        \STATE \quad \textbf{Method C:} Löwner--John ellipsoid sampling \CR{(higher cost, tightest empirical estimate; not strictly sound for finite $N$)}
        \STATE Propagate bounds through layers $2, \dots, k$ to obtain $[\ell^{(k)}, u^{(k)}]$
        \RETURN Certified output bounds $[\ell^{(k)}, u^{(k)}]$
    \end{algorithmic}
\end{algorithm}

Let \( f = f^{(k)} \circ \cdots \circ f^{(1)} \) denote a feedforward network with \( k \) layers, where \(f^{(1)}\) denotes the first affine layer (a spatio-temporal convolution), and each subsequent layer \(f^{(i)}\) for \(i \geq 2\) is either an affine transformation or an element-wise nonlinearity. For each neuron \( j \) in the first layer, we define the closed-form pre-activation bounds for \(f^{(1)}\) as $\ell_j^{(1)} := \min_{\mathbf{x}' \in T(\mathbf{x})} z_j^{(1)}(\mathbf{x}')$ and $u_j^{(1)} := \max_{\mathbf{x}' \in T(\mathbf{x})} z_j^{(1)}(\mathbf{x}')$, where \( z_j^{(1)}(\mathbf{x}') = \mathbf{w}_j^\top \mathbf{x}' + b_j \) is the affine response at unit \( j \). These bounds are computed exactly by solving the closed-form solution described in the previous section, and serve as certified enclosures of the neuron activations over the input specification. Our goal is to compute certified bounds on the network output \(f(\mathbf{x}')\) for all \(\mathbf{x}' \in T(\mathbf{x})\).

\textbf{Closed-Form Optimization on the First Layer} For each neuron \(j\), we consider the optimization problem $\max_{\tilde{\boldsymbol{\delta}}} / \min_{\tilde{\boldsymbol{\delta}}} \sum_{r=1}^{q} \alpha_{j,r} \tilde{\delta}_r \text{ s.t. } -\epsilon_r \le \tilde{\delta}_r \le \epsilon_r$.

\begin{theorem}[Closed-Form First-Layer Bounds]
	For each neuron \(j\) in the first layer, the exact bounds are given by
	\[
		\ell_j^{(1)} = z^{(1)}_{\mathrm{clean},j} - \sum_{r=1}^{q} \epsilon_r |\alpha_{j,r}|, \qquad
		u_j^{(1)} = z^{(1)}_{\mathrm{clean},j} + \sum_{r=1}^{q} \epsilon_r |\alpha_{j,r}|.
	\]
\end{theorem}
For completeness, the proof can be found in Appendix~\ref{app:proof-closedform}. In spatio-temporal convolutional layers, the coefficients \(\alpha_{j,r}\) correspond to effective convolutional weights obtained by summing kernel weights across temporal indices linked by \(\mathcal{S}\). The implementation computes these values efficiently using convolutions with absolute-valued kernels.

\subsubsection{Spatio-Temporal Bound Propagation Beyond the First Layer}
Exact analysis becomes intractable beyond the first layer due to non-linearities and dimensionality growth. We therefore propagate bounds forward using sound relaxation techniques. The closed-form bounds $[\ell^{(1)}, u^{(1)}]$ from the first layer serve as the anchor for three concrete verification pipelines:
\begin{enumerate}
    \item \textbf{STBP-IBP:} Closed-form first layer $\rightarrow$ standard interval bound propagation (fastest; axis-aligned boxes).
    \item \textbf{STBP-Lip:} Closed-form first layer $\rightarrow$ Lipschitz-based propagation via spectral norms ($\ell_2$-ball enclosure).
    \item \textbf{STBP-LJ:} Closed-form first layer $\rightarrow$ L\"owner--John ellipsoid sampling \CR{used as an \emph{empirical tightness oracle} for the reachable set} (\CR{tightest in practice; ellipsoidal enclosure with boundary sampling, not strictly sound for finite~$N$}).
\end{enumerate}
All three pipelines share the same exact first-layer analysis and differ only in how they propagate bounds through the remaining layers, trading off tightness against computational cost. \CR{Of the three, STBP-IBP and STBP-Lipschitz produce strictly sound certificates (Theorems~\ref{thrm:lipshitz-soundness} and standard IBP soundness); STBP-LJ is included as an empirical tightness reference for the reachable set (cf.\ Remark~\ref{rem:lj-empirical}).}

Let $[\ell^{(1)}, u^{(1)}]$ denote the certified bounds on the first-layer pre-activations obtained via the closed-form analysis. For all subsequent layers $i \geq 2$, we seek to compute bounds $[\ell^{(i)}, u^{(i)}]$ such that $f^{(i)} \circ \cdots \circ f^{(1)}(\mathbf{x}') \in [\ell^{(i)}, u^{(i)}] \quad \forall \mathbf{x}' \in T(\mathbf{x})$.
We now describe each approximation strategy.

\paragraph{Interval Bound Propagation (IBP)} IBP represents the reachable set at each layer as an axis-aligned hyperrectangle. Given certified bounds \([\ell^{(i-1)}, u^{(i-1)}]\), we define the center–radius representation: $\mathbf{c}^{(i-1)} = \tfrac{1}{2}(\ell^{(i-1)} + u^{(i-1)})$ and $\mathbf{r}^{(i-1)} = \tfrac{1}{2}(u^{(i-1)} - \ell^{(i-1)})$. For an affine layer $f^{(i)}(\mathbf{z}) = \mathbf{W}^{(i)} \mathbf{z} + \mathbf{b}^{(i)}$, bounds propagate as $\mathbf{c}^{(i)} = \mathbf{W}^{(i)} \mathbf{c}^{(i-1)} + \mathbf{b}^{(i)}$ and $\mathbf{r}^{(i)} = |\mathbf{W}^{(i)}| \mathbf{r}^{(i-1)}$. For ReLU non-linearities, bounds are updated element-wise using monotonicity: $\ell^{(i)} = \max(0, \mathbf{c}^{(i)} - \mathbf{r}^{(i)})$ and $u^{(i)} = \max(0, \mathbf{c}^{(i)} + \mathbf{r}^{(i)})$.

\paragraph{Lipschitz-Based Bound Propagation}

Instead of propagating coordinate-wise intervals, this method approximates global sensitivity via Lipschitz constants. For each affine layer, we compute a spectral-norm Lipschitz constant $L_i = \sigma(\mathbf{W}^{(i)})$, and accumulate $L_{\mathrm{total}} = \prod_{i=2}^k L_i$. Given bounds \([\ell^{(1)}, u^{(1)}]\), we propagate the center exactly while scaling the radius: $\mathbf{c}^{(i)} = f^{(i)}(\mathbf{c}^{(i-1)})$ and $\|\mathbf{r}^{(i)}\|_2 \le L_i \|\mathbf{r}^{(i-1)}\|_2$. ReLU layers preserve Lipschitz continuity with constant 1.

\begin{theorem}[Soundness of Lipschitz Bound Propagation after Closed-Form Analysis]
\label{thrm:lipshitz-soundness}
Let $f = f^{(k)} \circ \cdots \circ f^{(1)}$ be a neural network, where the first
layer $f^{(1)}$ admits an exact closed-form bound computation under an
$\ell_\infty$-bounded perturbation set $T(\mathbf{x}) = \{ \mathbf{x}' : \|\mathbf{x}' - \mathbf{x}\|_\infty \le \epsilon \}.$
Suppose the remaining subnetwork $\tilde{f} := f^{(k)} \circ \cdots \circ f^{(2)}$ is $L_{\mathrm{total}}$-Lipschitz with respect to the $\ell_2$-norm.Let $[\ell^{(1)}, u^{(1)}]$ denote the exact first-layer bounds obtained via the
closed-form analysis, and let $\mathbf{c}^{(1)} = \tfrac{1}{2}(\ell^{(1)} + u^{(1)}),$ and
$\mathbf{r}^{(1)} = \tfrac{1}{2}(u^{(1)} - \ell^{(1)});$
then, for all $\mathbf{x}' \in T(\mathbf{x})$, the network output satisfies
\[
f(\mathbf{x}') \in \left[ f(\mathbf{c}^{(1)}) \pm L_{\mathrm{total}}\sqrt{d}\,\|\mathbf{r}^{(1)}\|_\infty \mathbf{1} \right]
\]
where $d$ denotes the dimensionality of the input to $\tilde{f}$ and
$\mathbf{1}$ is the all-ones vector.
\end{theorem}
The full proof sketch of Theorem~\ref{thrm:lipshitz-soundness} is available in Appendix~\ref{app:proof-lipschitz}. A fundamental challenge in applying Lipschitz-based verification arises from the mismatch between the underlying normed spaces. While Lipschitz continuity is naturally defined with respect to the $\ell_2$-norm, the closed-form analysis of the first layer in our framework is derived under $\ell_\infty$-bounded perturbations. To maintain a consistent perturbation model and preserve soundness, we map the certified $\ell_\infty$-norm bounds obtained from the first layer into the $\ell_2$-norm space prior to Lipschitz bound propagation, and subsequently convert the resulting bounds back to the $\ell_\infty$ space. Although these norm conversions may slightly loosen the resulting bounds, they ensure that the verification procedure remains sound with respect to the original perturbation specification.

\paragraph{\CR{Löwner-John Ellipsoid Sampling (Empirical Tightness Oracle)}}

This approach represents the reachable set after the first layer as an ellipsoid $\mathcal{E} = \{ \mathbf{c}^{(1)} + \mathbf{A}\mathbf{u} : \|\mathbf{u}\|_2 \le 1 \},$ where $\mathbf{A}$ is derived from the first-layer radius. The Löwner-John ellipsoid is the minimum-volume ellipsoid enclosing the interval box.

The ellipsoid is propagated through the network by sampling points on its boundary, evaluating the network forward pass, and constructing an enclosing box over the sampled outputs. \CR{We use this procedure as an \emph{empirical tightness oracle}: it produces a statistical estimate of the reachable set $\{ f(\mathbf{z}) : \mathbf{z} \in \mathcal{Z}^{(1)} \}$ that approaches the true bounds as the number of boundary samples $N$ grows, but finite uniform sampling does not by itself certify enclosure of $f(\partial \mathcal{E})$ (see Remark~\ref{rem:lj-empirical}); STBP-IBP and STBP-Lipschitz remain our sound verification pipelines.} Algorithm~\ref{alg:lj} \CR{summarises the procedure}.
\begin{algorithm}[H]
    \caption{\CR{Löwner-John Ellipsoid Sampling (Empirical Tightness Oracle)}}
    \label{alg:lj}
    \begin{algorithmic}[1]
        \REQUIRE Center $\mathbf{c}^{(1)}$, radius $\mathbf{r}^{(1)}$, number of samples $N$
        \STATE Construct ellipsoid $\mathcal{E}$ enclosing $[\ell^{(1)}, u^{(1)}]$
        \FOR{$j = 1$ to $N$}
        \STATE Sample $\mathbf{u}_j$ uniformly from unit sphere
        \STATE Compute $\mathbf{z}_j = f(\mathbf{c}^{(1)} + \mathbf{A}\mathbf{u}_j)$
        \ENDFOR
        \STATE \CR{Compute coordinate-wise min/max of $\{\mathbf{z}_j\}_{j=1}^N$}
        \RETURN \CR{Empirical bound estimate}
    \end{algorithmic}
\end{algorithm}

\begin{theorem}[Soundness of Löwner--John Ellipsoid Propagation]
Let $f = f^{(k)} \circ \cdots \circ f^{(2)}$ denote the subnetwork following the
first layer, and let $[\ell^{(1)}, u^{(1)}]$ be the exact first-layer bounds
obtained via the closed-form analysis under an $\ell_\infty$-bounded perturbation
set $T(\mathbf{x})$. Let $\mathcal{Z}^{(1)} \subset \mathbb{R}^d$ denote the exact
reachable set of first-layer activations, and let $\mathcal{E}$ be the
Löwner--John ellipsoid of $\mathcal{Z}^{(1)}$, i.e., the minimum-volume ellipsoid
such that $\mathcal{Z}^{(1)} \subseteq \mathcal{E}.$ Assume that the bounds on the network output are 
computed by evaluating $f$ on a set of points whose convex hull contains $f(\partial \mathcal{E})$, where
$\partial \mathcal{E}$ denotes the boundary of $\mathcal{E}$. Then the resulting
output bounds constitute a sound enclosure of $\{ f(\mathbf{z}) : \mathbf{z} \in \mathcal{Z}^{(1)} \}$. A detailed proof is provided in Appendix~\ref{app:proof-ljsampling}.
\end{theorem}

\begin{remark}[Practical interpretation of STBP-LJ]
\label{rem:lj-empirical}
\CR{The theorem above is stated under the explicit hypothesis that the evaluated points convex-hull-cover $f(\partial \mathcal{E})$. With finite uniform boundary sampling, this hypothesis cannot in general be discharged a~priori, and the coordinate-wise min/max over a finite sample $\{f(\mathbf{z}_j)\}_{j=1}^N$ should therefore be interpreted as a \emph{statistical estimate} of the reachable output set rather than a certified enclosure. In this paper, we accordingly report STBP-LJ as an empirical reachable-set estimator that complements the strictly sound STBP-IBP and STBP-Lipschitz pipelines. 
}
\end{remark}

A challenge of Löwner–John ellipsoid–based sampling lies in the number of samples required to obtain \CR{an informative empirical estimate of} the reachable set. Naive sampling strategies over high-dimensional regions quickly become computationally prohibitive. By restricting sampling to the boundary of a tight Löwner–John ellipsoid that minimally encloses the reachable polytope, we substantially reduce the effective sampling domain while preserving coverage of extremal directions for \CR{the empirical estimate}.

At each layer \( i > 1 \), given lower and upper bounds \( \ell^{(i-1)} \) and \( u^{(i-1)} \) on the input to the layer, we compute valid enclosures for the output interval \( [\ell^{(i)}, u^{(i)}] \) using standard operations. For affine layers, given lower and upper bounds \( \ell^{(i-1)} \) and \( u^{(i-1)} \), we define the center and radius:
\[
\mathbf{c}^{(i-1)} := \tfrac{1}{2}\bigl(\ell^{(i-1)} + u^{(i-1)}\bigr), \quad \mathbf{r}^{(i-1)} := \tfrac{1}{2}\bigl(u^{(i-1)} - \ell^{(i-1)}\bigr).
\]
For an affine transformation weights \( \mathbf{W} \) and bias \( \mathbf{b} \), the output bounds are:
\[
\ell^{(i)} = \mathbf{W}\mathbf{c}^{(i-1)} - |\mathbf{W}|\mathbf{r}^{(i-1)} + \mathbf{b}, \quad u^{(i)} = \mathbf{W}\mathbf{c}^{(i-1)} + |\mathbf{W}|\mathbf{r}^{(i-1)} + \mathbf{b}.
\] where the absolute value is taken element-wise.

\looseness=-1
The distinguishing feature of STBP is that it anchors the verification procedure with the closed form solution on the first layer, while relying on fast and conservative propagation for the remainder of the network using either of the 3 approximation strategies. \CR{Each method trades off tightness and computational cost; STBP-IBP and STBP-Lipschitz preserve formal soundness, while STBP-LJ provides an empirical tightness reference at higher sample cost.} This approach balances precision and scalability, and leverages the structure of spatio-temporal input domains to tighten the relaxation at the input stage. We prove that under normal conditions, our approach should be strictly tighter than naive IBP (defined and proved in Appendix~\ref{app:proof-ibp-greater-than-closedform}).

\paragraph{Generalization to Multiple Affine Layers.}
Although we apply the closed-form solution to the first convolutional layer in our primary experiments, the technique is not inherently limited to a single layer. For any sequence of consecutive affine layers without intervening non-linearities (e.g., Conv$\rightarrow$BatchNorm, when BatchNorm is folded into the convolution), the composition $\mathbf{W}_2(\mathbf{W}_1\mathbf{x} + \mathbf{b}_1) + \mathbf{b}_2$ remains affine, and the shared-perturbation structure is preserved through the mapping. More generally, the closed-form analysis can be extended to the first $k$ affine layers by composing the weight matrices and re-applying the closed-form optimization to the composite linear map. In practice, the tightness benefit of the closed-form is greatest at the first layer, where the spatio-temporal constraints are most directly exploitable and where standard IBP introduces the most over-approximation. We empirically validate this observation in Section~\ref{subsection:perlayerboundwidth} through a per-layer ablation study that compares the bound widths produced by pure IBP and STBP-IBP at every layer of the network, confirming that the first convolutional layer is the dominant contributor to bound tightness and that the advantage propagates through all subsequent layers.

\paragraph{Computational Complexity.}
Table~\ref{tab:complexity} summarises the time complexity of each STBP component, where $m$ is the number of first-layer neurons, $q$ the number of independent perturbation variables, $n_i$ the width of layer $i$, and $N_s$ the number of Löwner--John samples.

\begin{table}[H]
\centering
\small
\begin{tabular}{@{}lll@{}}
\toprule
Component & Time & Space \\
\midrule
Closed-form (layer 1) & $\mathcal{O}(m \cdot q)$ & $\mathcal{O}(m \cdot q)$ \\
IBP propagation & $\mathcal{O}(\sum n_i n_{i+1})$ & $\mathcal{O}(\max n_i)$ \\
Lipschitz propagation & $\mathcal{O}(\sum n_i^2)$ & $\mathcal{O}(\max n_i^2)$ \\
LJ Ellipsoid Sampling & $\mathcal{O}(N_s \cdot \text{fwd})$ & $\mathcal{O}(N_s \cdot n_k)$ \\
\bottomrule
\end{tabular}
\caption{Per-sample time and space complexity of STBP components. \CR{Here $m$ is the number of first-layer neurons, $q$ the number of independent perturbation variables, $n_i$ the width of layer $i$, $N_s$ the number of Löwner--John samples, and $\mathrm{fwd}$ the cost of a single forward pass through the subnetwork $f^{(k)} \circ \cdots \circ f^{(2)}$.}}
\label{tab:complexity}
\end{table}

%% file: sections/5_experiments.tex
\section{Evaluation}
\label{sec:evaluation}
Evaluations on the robustness verification for spatio-temporal systems are described. We consider 3D Convolution based models on various sizes, with the largest model being a custom trained ResNet-18 with 3D Convolutions for videos. All the experiments were conducted on the machine having an AMD EPYC 9334 32-Core CPU with 2 X Nvidia L40 GPUs, each having 48GB GPU memory running on Ubuntu 22.04 and kernel version 6.8.0.

\subsection{Datasets and Models}
\label{subsec:datasets}
\begin{enumerate}[
    leftmargin=*,
    itemsep=0pt,
    parsep=0pt,
    topsep=2pt,
    partopsep=0pt
]
	\item \textbf{MNIST~\cite{lecun2005MNIST}:} We construct a video-based toy dataset using the MNIST digit recognition dataset.
    To simulate spatio-temporal structure, we concatenate 10 consecutive digit frames per video and apply fixed transformations to each frame sequence. We evaluate two 3D convolutional models: one with an input resolution \(8 \times 8\) with 5 frames and the other with the resolution \(28 \times 28\) with 10 frames. Both models have 9 layers with 3 convolutions, 4 ReLU activations, and 3 fully connected layers. The models achieve clean classification accuracies of 93.10\% and \CR{91.30\%}, respectively.
	\item \textbf{UCF-101~\cite{soomro2012ucf101}:} UCF-101 is a benchmark dataset for video activity recognition, comprising 13,320 YouTube video clips across 101 classes, grouped into five action categories. 
	Each video has a resolution of \(320 \times 240\) pixels. A 3D CNN network with 9 layers—consisting of 2 convolutional layers, 4 ReLU activations, and 3 fully connected layers, trained on the dataset using the 30 sampled frames from each video at an input resolution of \(256 \times 342\) yields a classification accuracy of 50.842\%. To enhance both performance and scalability, the model input resolution was downsampled to \(32 \times 32\) and we restric to $5$ classes, yielding an accuracy of 74.41\%. 
	\item \textbf{Udacity Steering Angle Prediction dataset~\cite{udacity2017udacity}:} We downsample images from the steering angle prediction dataset from \(640 \times 480\) to \(32 \times 32\) and cast the angle prediction class as a three-class classification. Training a model comprised of 10 layers, including 2 convolutional layers, 3 ReLU activations, an adaptive average pooling layer, and 3 fully connected layers with a clean accuracy of 84.239\%. 
	\item \textbf{MEDMNIST~\cite{yang2021medmnistv2}:} Using the same architecture as the MNIST example above, we focus on the Synapse-3D (brain imaging) dataset  from MEDMNIST downsampled to voxels of size \(32 \times 32 \times 32\). On this binary task (excitatory/inhibitory classification) the network achieves 73.01\% accuracy. 
    \item \CR{\textbf{GTSRB~\cite{stallkamp2011gtrsb}:} For direct comparability with VideoStar~\cite{sasaki2025robustness} we evaluate on the German Traffic Sign Recognition Benchmark. Following the VideoStar protocol, we construct 16-frame videos by repeating each traffic-sign image with small affine jitter to simulate camera motion across frames. Inputs are resized to a fixed resolution and we use a compact 3D CNN that matches the architecture-class of VideoStar's reference model; training on the resulting video dataset yields a clean accuracy of 71.23\% (cf.\ 87.5\% reported by VideoStar on still images). Verification uses 215 samples following the original PRSV protocol (see Section~\ref{subsection:VideoStarComparison}).}
\end{enumerate}

\subsection{ST-Bench}
\label{subsec:ST-Bench}
We propose a~\href{https://zenodo.org/records/15426146}{benchmark} for spatio-temporal verification with realistic perturbations based on real-world scenarios. To obtain such perturbations, we use the YOLOv12~\cite{tian2025yolov12} object detection model and generate bounding boxes for objects within the frames. A JSON file is generated by track the movement of various objects. The file contains each tracked object with its corresponding image frames and the bounding box information. 
As opposed to applying adversarial perturbations on the entire sequence of frames, the perturbations are applied to the bounding boxes. This models a real world perturbation and allows better scalability of the verification bounds. 

\subsection{Empirical Evaluation on MNIST toy dataset}
\label{subsec:STBP-evaluation}
To evaluate the efficiency of STBP approach and to identify the best approximation for the subsequent layers of the network, we run experiments to verify the robustness of the MNIST toy video dataset described in~\ref{subsec:datasets} with 10 frames, each frame having 28 x 28 dimension. The input perturbations $\epsilon$ varies from $[0.0001, 3]$. To keep computational overhead minimal, the Löwner-John Sampling approach limits the number of samples to 75 data points on the ellipsoid boundary, per data input. We obtain the following plot as described in Figure~\ref{fig:mnist-stbp-eval}. We see that while using IBP bounds after the first layer provides tighter bounds than IBP on the entire network, the closed form solution with Lipschitz approximation on the subsequent layers is significantly tighter. Furthermore, \CR{the Löwner–John ellipsoid sampling procedure produces a much tighter \emph{empirical} estimate of the reachable set, at the cost of sampling and runtime; we treat this as an oracle on the achievable tightness rather than a certified result (cf.\ Remark~\ref{rem:lj-empirical}).} Overall, the Lipschitz approximation is found to be \CR{the most effective sound verifier in this study and works well with networks optimized for Lipschitz continuity}.

\subsection{Comparison with VideoStar}
\label{subsection:VideoStarComparison}

\begin{table*}[!htbp]
    \begin{center}
        \resizebox{\textwidth}{!}{%
            \begin{tabular}{@{}llcccccccc@{}}
                \toprule
                \multirow{2}{*}{Experiment} & \multirow{2}{*}{Clean Accuracy} & \multicolumn{3}{c}{Model Parameters} & \multicolumn{5}{c}{IBP Rob. Acc. (\%)} \\
                \cmidrule(l){3-5} \cmidrule(l){6-10}
                 & & \multicolumn{1}{l}{Input dim} & \multicolumn{1}{l}{Output dim} & \multicolumn{1}{l}{No. of Samples} & \multicolumn{1}{l}{$\epsilon=0.1$} & \multicolumn{1}{l}{$\epsilon=0.01$} & \multicolumn{1}{l}{$\epsilon=0.001$} & \multicolumn{1}{l}{$\epsilon=0.0001$} & \multicolumn{1}{l}{$\epsilon=0.00001$} \\
                 \midrule
                MNIST Toy Model & 93.1\% & 1 x 5 x 8 x 8 & 10 & 1000 & 0.0 & 0.0 & 1.2 & 53.2 & 89.1 \\
                MNIST Toy Model Large & 91.3\% & 1 x 10 x 28 x 28 & 10 & 1000 & 0.0 & 0.0 & 53.2 & 69.1 & 71.0 \\
                UCF-101 & 74.41\% & 3 x 30 x 32 x 32 & 5 & 500 & 0.0 & 0.0 & 0.0 & 20.03 & 26.36 \\
                Udacity Steering Angle & 84.23\% & 3 x 30 x 32 x 32 & 3 & 1123 & 0.0 & 0.0 & 0.0 & 70.0 & 75.0 \\
                MEDMNIST Synapse3D & 72.88\% & 32 x 32 x 32 & 2 & 177 & 0.0 & 0.0 & 19.77 & 72.88 & 72.88 \\
                MEDMNIST Synapse3D & 73.01\% & 64 x 64 x 64 & 2 & 352 & 0.0 & 0.0 & 0.0 & 73.01 & 73.01 \\ \bottomrule
            \end{tabular}%
        }
    \end{center}
    \caption{Clean accuracy, model parameters, and certified robustness of spatio-temporal models using IBP.}
    \label{tab:model-params}
\end{table*}

\begin{figure*}[htbp]
\centering
\begin{tabular}{cc}
\fbox{\includegraphics[alt={Certified robust accuracy of IBP, STBP-IBP, STBP-Lipschitz, and STBP-Löwner-John on an MNIST video model (10 frames, 28x28) as a function of perturbation magnitude epsilon; STBP variants substantially outperform IBP, with STBP-LJ tightest at high cost.}, width=0.43\textwidth]{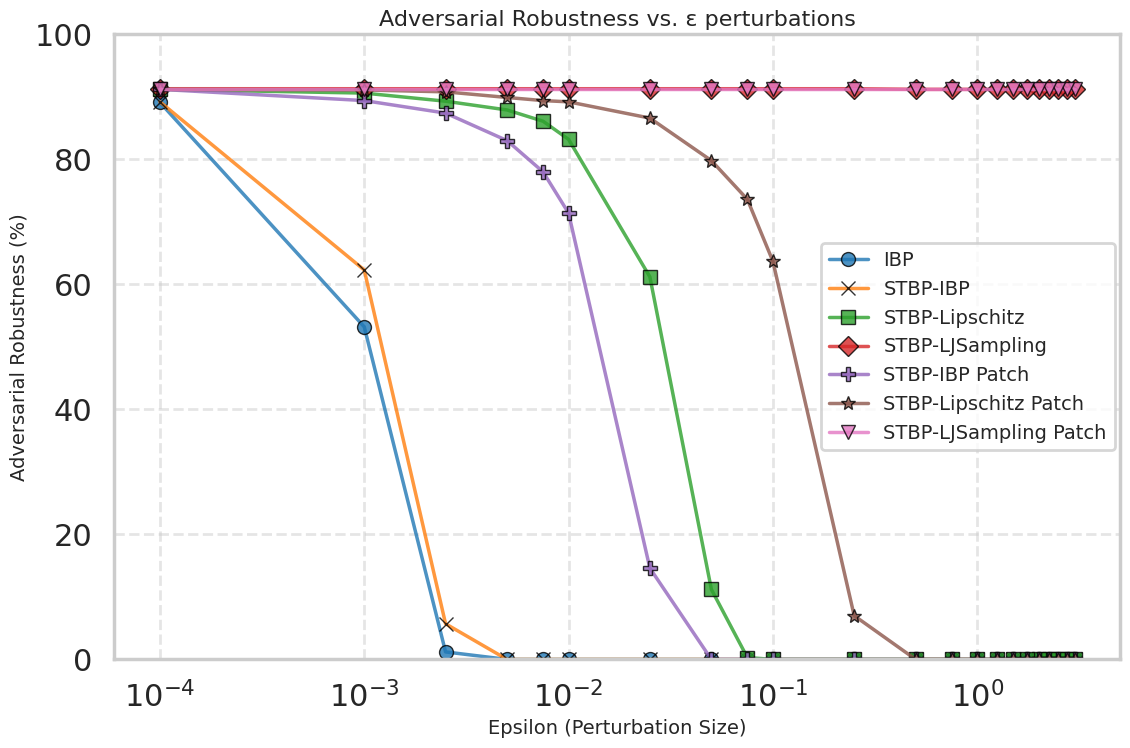}} &
\fbox{\includegraphics[alt={Certified robust accuracy of IBP and STBP variants on the same MNIST video model as a function of adversarial patch size k; the patch-constrained STBP variant retains near-perfect robustness for small patches.}, width=0.47\textwidth]{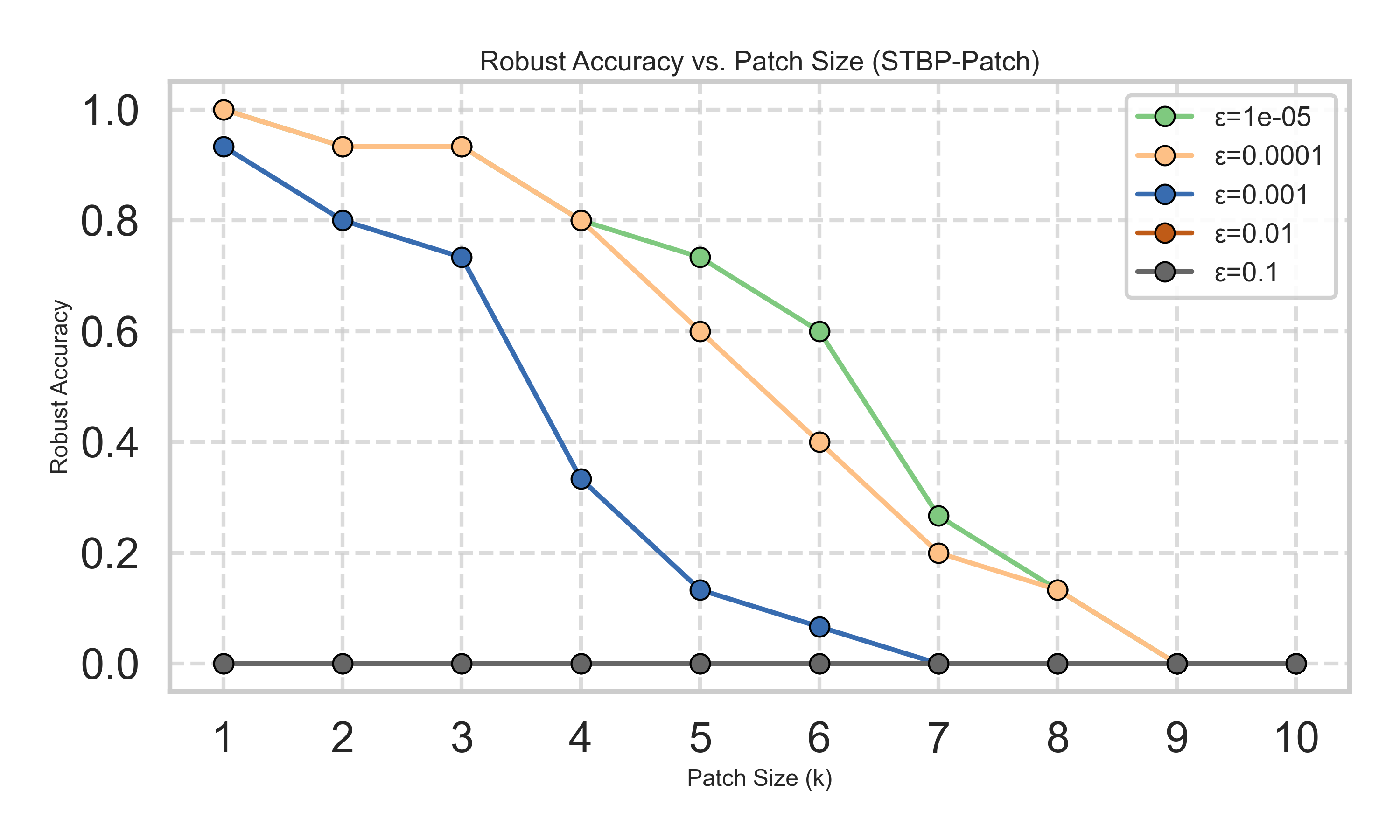}} \\
(a) & (b)
\end{tabular}
\caption{Adversarial robustness of IBP, STBP-IBP, STBP-Lipschitz, and
STBP-L\"owner--John Sampling using adversarial patches for an MNIST video model
with 10 frames of size $28 \times 28$:
(a) robustness against perturbation magnitude ($\epsilon$);
(b) robustness against patch size ($k$).}
\label{fig:mnist-stbp-eval}
\end{figure*}
VideoStar represents reachable sets using symbolic tuples $(c, V, P, l, u)$, where $c$ denotes an anchor video, $V$ a set of generator videos, and $P$ linear constraints on combination coefficients; bounds are propagated by solving linear programs (LPs) for each neuron at every layer. While this formulation supports general linear perturbation constraints, it incurs substantial computational overhead, requiring $2m$ LP solves per layer, which becomes prohibitive for high-dimensional video inputs. In contrast, our method exploits the structure of shared spatio-temporal perturbations to derive a closed-form solution for the first convolutional layer, eliminating LP solving entirely. In the common setting where perturbations are shared across temporal frames, the effective weights reduce to a temporal sum, enabling direct computation of tight bounds via $z_{\mathrm{upper}} = z_{\mathrm{clean}} + \varepsilon \sum |W_{\mathrm{eff}}|$, which is mathematically equivalent to the LP optimum but reduces complexity from $\mathcal{O}(m \cdot \mathrm{poly}(n))$ to $\mathcal{O}(\text{forward\_pass})$ and allows full GPU parallelization. Empirically, we compare both approaches on the GTSRB dataset using 16-frame videos and report Proportion of Robust Samples Verified (PRSV) and runtime over 215 samples, following the VideoStar evaluation protocol. Despite a lower clean accuracy (71.23\% versus 87.5\% reported by VideoStar), our STBP implementation with Löwner–John sampling achieves higher PRSV and significantly lower runtime, as shown in Table~\ref{tab:videostar-stbp}. \CR{We report STBP-LJ in this comparison because it is the tightest of the three STBP variants in our experiments and therefore the most demanding configuration in runtime; PRSV values for STBP-LJ correspond to the empirical tightness oracle of Section~\ref{sec:method} (Remark~\ref{rem:lj-empirical}), while STBP-IBP and STBP-Lipschitz provide sound certificates on the same specifications at comparable or lower runtime.}

\begin{table}[t]
\centering
\scriptsize
\setlength{\tabcolsep}{4pt} 
\begin{tabularx}{\columnwidth}{@{}p{0.32\columnwidth} p{0.22\columnwidth} c c c@{}}
\toprule
\scriptsize
Method & Metric & \multicolumn{3}{c}{Epsilon} \\
 &  & 1/255 & 2/255 & 3/255 \\
\midrule
VideoStar (Relax)$^{*}$ 
& PRSV (\%)        
& 91.16 & 74.88 & 57.21 \\

& Avg.Runtime(s) 
& 222.43 & 459.89 & 1044.40 \\

STBP (LJ Sampling) 
& PRSV (\%)        
& \textbf{99.98} & \textbf{93.1} & \textbf{91.2} \\

& Avg.Runtime(s) 
& \textbf{0.10} & \textbf{0.13} & \textbf{0.17} \\
\bottomrule
\end{tabularx}
\caption{Comparison of VideoStar against STBP for the GTSRB dataset.
$^{*}$VideoStar results are taken from~\cite{sasaki2025robustness}; see text for discussion on hardware comparability.}
\label{tab:videostar-stbp}
\end{table}

\looseness=-1
\noindent\textbf{Remark on hardware comparability.} VideoStar runtimes are as reported in~\cite{sasaki2025robustness}. While the experiments were not conducted on identical hardware, the observed runtime gap exceeds three orders of magnitude (e.g., 222s vs. 0.1s at $\epsilon{=}1/255$), which cannot be attributed to hardware differences alone. The fundamental source of this gap is algorithmic: VideoStar requires $2m$ LP solves per layer per sample, whereas STBP computes bounds in a single forward pass.

\subsection{Experiments}
\label{subsection:Evaluations}

\begin{table*}[!htbp]
    \begin{center}
        \resizebox{\textwidth}{!}{%
            \begin{tabular}{@{}lcccccccccc@{}}
                \toprule
                \multirow{2}{*}{Experiment} & \multicolumn{5}{c}{CROWN-IBP Rob. Acc. (\%)} & \multicolumn{5}{c}{STBP-Lipschitz Adv. Patch Rob. Acc. (\%)} \\
                \cmidrule(l){2-6} \cmidrule(l){7-11}
                 & \multicolumn{1}{l}{$\epsilon=0.1$} & \multicolumn{1}{l}{$\epsilon=0.01$} & \multicolumn{1}{l}{$\epsilon=0.001$} & \multicolumn{1}{l}{$\epsilon=0.0001$} & \multicolumn{1}{l}{$\epsilon=0.00001$} & \multicolumn{1}{l}{$\epsilon=0.1$} & \multicolumn{1}{l}{$\epsilon=0.01$} & \multicolumn{1}{l}{$\epsilon=0.001$} & \multicolumn{1}{l}{$\epsilon=0.0001$} & \multicolumn{1}{l}{$\epsilon=0.00001$} \\
                 \midrule
                MNIST Toy Model & 0.0 & 17.9 & 69.96 & 82.0 & 82.25 & \textbf{0.4} & \textbf{20.45} & \textbf{79.5} & \textbf{92} & \textbf{93.10} \\
                MNIST Toy Model Large & 0.0 & 15.8 & 87.4 & 88.2 & 88.3 & \textbf{63.7} & \textbf{89.2} & \textbf{91.1} & \textbf{91.2} & \textbf{91.3} \\
                UCF-101 & 0.0 & 0.0 & 2.07 & 49.93 & 63.48 & \textbf{2.6} & \textbf{12.3} & \textbf{19.36} & \textbf{61.9} & \textbf{71.81} \\
                Udacity Steering Angle & 0.0 & 0.0 & 5.61 & \textbf{84.24} & \textbf{84.24} & \textbf{7.24} & \textbf{84.24} & \textbf{84.24} & \textbf{84.24} & \textbf{84.24} \\
                MEDMNIST Synapse3D & 0.0 & 0.0 & 71.75 & \textbf{72.88} & \textbf{72.88} & 0.0 & \textbf{72.88} & \textbf{72.88} & \textbf{72.88} & \textbf{72.88} \\
                MEDMNIST Synapse3D & 0.0 & 0.0 & 8.24 & \textbf{73.01} & \textbf{73.01} & 0.0 & 0.0 & \textbf{33.33} & \textbf{73.01} & \textbf{73.01} \\ \bottomrule
            \end{tabular}%
        }
    \end{center}
    \caption{Certified robustness of spatio-temporal models using CROWN-IBP and STBP-Lipschitz approximations under adversarial patch perturbations.}
    \label{tab:results}
\end{table*}

\looseness=-1
We empirically validate the effectiveness of our approach towards certifying adversarial robustness of neural networks finding marked improvement in certification on each of the tested datsets and models. We assess robustness against perturbation magnitudes $\epsilon \in \{10^{-5}, 10^{-4}, 10^{-3}, 10^{-2}, 10^{-1}\}$ and patch sizes $k \in \{1, \dots, 10\}$, comparing standard IBP, our hybrid STBP approach. We omit CROWN from our experimental comparisons because the extensive memory requirements for its backward propagation consistently result in Out-of-Memory (OOM) errors, even on the smallest MNIST toy model.

\looseness=-1
Figure~\ref{fig:mnist-stbp-eval} illustrates the certified robust accuracy of the MNIST video model under both increasing perturbation magnitude (Figure~2a) and patch size (Figure~2b). STBP consistently demonstrates better performance than IBP across all perturbation levels. Notably, the patch-based variant of STBP achieves nearly perfect robustness up to $\epsilon = 10^{-4}$ and small patch sizes, highlighting the effectiveness of incorporating realistic spatio-temporal constraints.

\looseness=-1
The results are quantitatively summarized in the Table~\ref{tab:results}, which reports clean and certified robust accuracies across four domains: MNIST, UCF-101, Udacity self-driving, and MEDMNIST Synapse3D. STBP achieves substantial improvements in robust accuracy over IBP—most notably on the MNIST model, where STBP with patch constraints improves robustness from 49.32\% to 77.05\% at $\epsilon = 10^{-4}$. For autonomous driving and medical imaging datasets, STBP yields over 85\% robust accuracy under patch perturbations, demonstrating its scalability to real-world spatio-temporal systems.

\looseness=-1
While our approach scales favorably for moderate-sized 3D CNNs, larger architectures such as VLMs remain challenging even under shared perturbation constraints. Full runtime comparisons and analysis of failure cases are presented in the Appendix~\ref{app:resultsandruntimes}. Although the STBP approach with Löwner-John sampling incurs additional memory due to sampling and ellipsoid approximation, this cost is justified by the \CR{significantly improved empirical tightness it provides as a reachable-set oracle}. \CR{All three approaches have runtimes that are comparable to IBP and yield tighter bounds than IBP—soundly so for STBP-IBP and STBP-Lipschitz, and empirically so for STBP-LJ.}

\subsection{\CR{On the Magnitude of Verifiable $\epsilon$}}
\label{subsec:epsilon-magnitude}
\looseness=-1
\CR{The certifiable perturbation magnitudes reported across our experiments (e.g., $\epsilon \in [10^{-5}, 10^{-1}]$ on the MNIST toy video, and as small as $10^{-4}$ on MEDMNIST Synapse3D) are substantially smaller than those typically reported for single-image $\ell_\infty$ verification on flattened MNIST under matched threat models. This is a structural consequence of the spatio-temporal verification setting rather than a weakness of STBP, and we summarise the contributing factors below.}

\looseness=-1
\CR{\emph{(i) Per-element $\epsilon$ defines a strictly larger input set in the temporal regime.} For a video with $T$ frames, an $\ell_\infty$ ball of radius $\epsilon$ encloses all per-pixel-per-frame perturbations of magnitude at most $\epsilon$; the volume of the admissible input set scales as $(2\epsilon)^{C \cdot D \cdot H \cdot W}$ rather than $(2\epsilon)^{C \cdot H \cdot W}$ for a single image. Certifiable $\epsilon$ per element naturally shrinks with the number of frames, all else equal.}

\looseness=-1
\CR{\emph{(ii) Structured, correlated perturbations are a richer threat model.} The shared and patch-based constraints we verify against impose correlation between perturbations across frames, 
yielding a strictly more powerful adversary than the i.i.d.\ per-pixel attacker typically assumed in flattened-MNIST benchmarks. Smaller certifiable $\epsilon$ values are the expected price of certifying against richer adversary.}

\looseness=-1
\CR{\emph{(iii) 3D convolutions and temporal coupling loosen bound propagation.} Compared to MLP-style flattened-MNIST baselines, 3D CNNs have deeper effective receptive fields and intertwined spatio-temporal weights, both of which amplify the over-approximation inherent in any sound relaxation.}

\looseness=-1
\CR{\emph{Interpretation of the results.} The relevant comparison is therefore the \emph{relative} tightness of STBP against IBP / CROWN-IBP under the same spatio-temporal threat model, not the absolute scale of $\epsilon$. Under this lens, STBP yields up to $1.7\times$ improvement in certified robust accuracy over IBP at matched $\epsilon$, demonstrating substantive practical value in a regime where prior methods produce vacuous or trivial bounds.}

\subsection{Ablation: Per-Layer Bound Width}
\label{subsection:perlayerboundwidth}

\begin{figure*}[t]
\centering
\fbox{\includegraphics[alt={Per-layer bound width comparison between vanilla IBP and STBP-IBP, illustrating that the closed-form first-layer analysis tightens bounds at the input stage and that the tightness advantage propagates through all subsequent layers.}, width=0.9\textwidth]{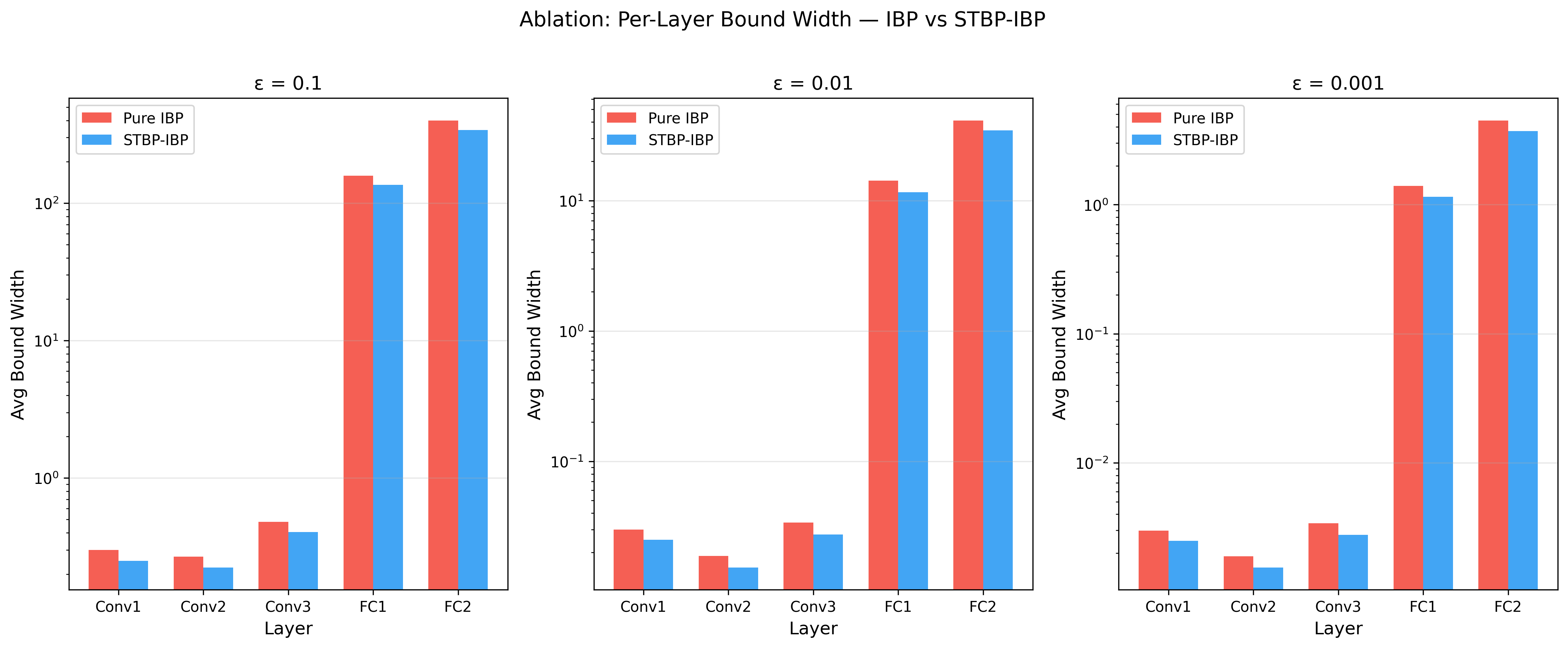}}
\caption{Per-layer average bound width for pure IBP versus STBP-IBP on the MNIST video model (10 frames, $28 \times 28$) at 3 perturbation magnitudes ($\epsilon \in \{0.1, 0.01, 0.001\}$). STBP-IBP consistently produces tighter bounds at every layer, with the tightness advantage established at Conv1 propagating through Conv2, Conv3, FC1, and FC2.}
\label{fig:per-layer-bound-width}
\end{figure*}

\looseness=-1
To understand \emph{where} and \emph{how} the STBP bound tightness arises, we conduct a per-layer ablation comparing the average bound width produced by pure IBP and STBP-IBP through each layer of the MNIST video model. Figure~\ref{fig:per-layer-bound-width} plots the results across three perturbation magnitudes.

\looseness=-1
At the first convolutional layer (Conv1), the closed-form STBP solution exploits the shared spatio-temporal perturbation structure to reduce the effective weight contribution from $W_{k,c,d,h,w}$ to $W_{\mathrm{eff}} = \sum_d W_{k,c,d,h,w}$, yielding bounds that are approximately $1.20\times$ tighter than standard IBP. Crucially, this tightness advantage \emph{propagates} through all subsequent layers. As the tighter Conv1 output intervals pass through ReLU activations and subsequent affine layers via standard IBP, the bound widths at Conv2, Conv3, FC1, and FC2 remain consistently $1.16$--$1.23\times$ smaller than their pure-IBP counterparts, across all tested $\epsilon$ values.

\looseness=-1
Although the closed-form derivation in Section~\ref{sec:method} generalises to any affine layer---not only the first convolution---applying it beyond the first layer yields \emph{no additional tightening}. This is because, after the first ReLU activation, each neuron's output interval is independently bounded and no longer shares the structured perturbation pattern exploited by the closed form. Consequently, the subsequent-layer closed-form bounds reduce to standard IBP bounds. The first convolutional layer is therefore the \emph{optimal and sufficient} point at which to apply the STBP closed-form analysis: it captures the full benefit of the shared perturbation structure while incurring only a single forward-pass overhead that is negligible compared to the overall verification runtime.

%% file: sections/6_conclusion.tex
\section{Conclusions}
\label{sec:conclusion}
Robustness verification remains a significant challenge for spatio-temporal systems due to the high dimensionality of both their input and model architectures. This work addresses this limitation by introducing a scalable verification framework that exploits spatio-temporal structure through an exact closed-form solution for the first convolutional layer, combined with scalable bound propagation for subsequent layers. Our experimental results show notable improvements in certified robustness across diverse spatio-temporal tasks, with robust accuracy gains of up to $1.7\times$ compared to IBP under identical perturbation budgets. To facilitate systematic evaluation, we introduce \texttt{ST-Bench}, a benchmark suite specifically designed for verifying robustness in spatio-temporal systems. Using ST-Bench, we demonstrate how the proposed verification method can be effectively scaled to real-world scenarios.

\textit{\CR{Limitations.}}
\CR{Although STBP offers improvements in spatio-temporal verification, several open problems remain. First, while STBP-IBP and STBP-Lipschitz provide strictly sound certificates, the Löwner--John pipeline (STBP-LJ) relies on finite uniform boundary sampling and is reported as an empirical reachable-set estimator rather than a certified verifier (cf.\ Remark~\ref{rem:lj-empirical} in Section~\ref{sec:method}); a deterministic sound variant—e.g., via Lipschitz inflation of the sampled extrema, deterministic spherical codes, or vertex enumeration of the first-layer polytope—is a natural direction for future work. Second, although STBP scales favourably to moderate-sized 3D~CNNs, formal verification of considerably larger architectures such as VLMs and large vision–video transformers remains an open problem for the field as a whole: closed-form bounds for self-attention and sound relaxations for softmax / layer normalisation are still missing, and our framework inherits this limitation. Third, the verifiable perturbation magnitudes reported in our experiments are necessarily smaller than those typically reported for single-image $\ell_\infty$ verification under matched threat models, because shared and correlated cross-frame perturbations define a strictly richer adversarial set; STBP's contribution should therefore be read as a \emph{relative} improvement (up to $1.7\times$ over IBP) in this harder regime rather than an absolute statement about achievable $\epsilon$.}

Future research directions include the development of automated mechanisms for generating selective constraints and heuristic strategies to enable provable robustness in large-scale, real-world spatio-temporal AI systems. Another promising avenue lies in integrating Spatio-Temporal Bound Propagation (STBP) with certified training pipelines to further enhance robustness guarantees.

%% file: sections/7_acknowledgements.tex
\begin{credits}
\subsubsection{\ackname}
\label{sec:acknowledgements}
Sherwin Varghese is supported by the UKRI Horizon EU project
"EVENFLOW" (Project ID: 101070430). Alessio Lomuscio is supported by a Royal Academy of Engineering Chair in Emerging Technologies.

\end{credits}

%% file: sections/9_appendix.tex
\clearpage
\appendix
\section{Appendix}

\subsection{Preliminaries}
\label{app:preliminaries}

\subsubsection{Neural Network}
We consider a Feed-Forward Neural Networks (FFNN) that perform classification of temporal sequences. In particular, our focus is on video classification systems that takes in a set of input video frames as input. Formally, we define the network $f_{\theta}$ with parameters $\theta$ as $f_{\theta}: \mathbb{R}^{N \times C \times W \times H} \rightarrow \mathbb{R}^{M}$ where $N$ denotes the number of frames per video, $H$ denotes the height, $W$ denotes the width of each input frame, $C$ denotes the channels in of each image frame (for RGB, $C$=3) and $M$ denotes the number of output classes. The network assigns a class $i \in \{1, 2, \dots, M\}$ if $f_{\theta}(x)_i > f_{\theta}(x)_j\ \forall i \neq j$.

\subsubsection{Neural Network Classification}
Given the FFNN $f_{\theta}$ and the input video (represented by several image frames)  $X$, the clean prediction of the network is: $\hat{y} = f_{\theta}(x)$.
The FFNN is composed of intermediate layers. We denote $x_0$ as the input layer. The final output of the network is denoted by $y_\textsubscript{target}$. The neural network is composed of a series of transformations $h_l$ within each of its $L$ layers. Thus, the output of each layer $z_l$ can be expressed by the following equation:
\begin{equation}
    z_l = h_l(x_\textsubscript{$l-1$}) \quad l = 1, 2, ... L
\end{equation}
The transformations $h_l$ can be an affine transformation or a monotonic function such as ReLU or sigmoid. We represent the affine transformation as 
\begin{equation}
    h_l(x_\textsubscript{$l-1$}) = W_l\ \cdot  x_{l-1} + b_l
\end{equation}
where $W_l \in \mathbb{R}^{N \times C \times W \times H}$ is the weight of the matrix and $b_l$ is the bias vector for the layer $l$.

\subsubsection{Robustness Verification}
Verification techniques validate whether neural networks satisfy certain specifications. In this setting, we focus on ensuring that the adversarial robustness specification holds.

In order to verify the robustness of $f_{\theta}$ against adversarial inputs, we apply perturbations on the input video $X \in \mathbb{R}^{N \times C \times W \times H}$, where $N$ denotes the total number of frames and $C$ denotes the number of color channels. We observe the outputs of $f_{\theta}$ under the adversarial perturbation defined by the budget $\epsilon$. Thus, verification proceeds by searching for a counter example that violates the specification or until all the samples satisfy the specification constraint.
\\
As defined by \cite{gowal2018ibp}, the network is called robust at a point $x_0$ if no adversarial example can cause the prediction of the network to change from its target label for all $z_0 \in \chi(x_0)$, where $\chi(x_0)$ is a set of inputs around $x_0$. This is expressed formally for each class $y$ as:
\begin{equation}
    (e_{\hat{y}'} - e_{\hat{y}})^T z_l \leq 0 \quad \forall z_0 \in \chi(x_0) = \{ x | \norm{x - x_0}_\infty < \epsilon \}
\end{equation}

\subsubsection{Adversarial Perturbation}

We apply a perturbation $\delta$ bounded by $\epsilon$ to the input video frames $X$ within a norm $\ell_\infty $-norm or $ \ell_2 $-norm $\epsilon \geq 0$, \ $\epsilon$ being the maximum allowed perturbations of each frame in the video.
For the adversarial input $X^\text{adv} = X + \delta$, the output of the network~\cite{goodfellow2014explaining} is defined by $\hat{y}' = f_{\theta}(X^\text{adv}).$
For a robust network, the input with perturbations $X +\delta$ maintains the same prediction $\hat{y}' = \hat{y}$

\subsubsection{Projected Gradient Descent (PGD) Attack}
To falsify the robustness of the neural network, we use Projected Gradient Descent as the adversarial attack strategy. The PGD algorithm iteratively applies gradient-based perturbations to maximize the network's loss and projects the perturbed input back into the feasible region defined by the perturbation bound.

Mathematically, PGD Attack is defined as -
\begin{equation*}
\delta^{(t+1)} = \Pi_\epsilon(\delta^{(t)} + \alpha \cdot \text{sign}(\nabla_X \mathcal{L}(f_{\theta}(X + \delta), y)))
\end{equation*}
    
The projection operator $\Pi_\epsilon$ ensures that the perturbation remains within the allowable budget:
\begin{equation}
    \begin{aligned}[t]
    \Pi_\epsilon(\delta) =& \arg \min_{X^\text{adv}} \norm{X - X^\text{adv}} 
    \\ & \  \text{subject to} \quad \norm{X - X^\text{adv}} \leq \epsilon.
    \end{aligned}
\end{equation}
For the $\ell_\infty$-norm, this corresponds to clipping the perturbation to lie within the range $[-\epsilon, \epsilon]$ for each input.

\subsubsection{Adversarial Training}
Adversarial training is a widely adopted defence technique designed to improve the robustness of deep neural networks against adversarial perturbations. The core idea is to augment the training process by explicitly incorporating adversarial examples into the optimization objective. This ensures that the model not only performs well on clean data but also makes it robust against perturbations .

The adversarial training objective is defined using a weighted combination of clean and adversarial losses. Let \( X \) denote the original input, \( X^{\text{adv}} \) the adversarially perturbed input, \( y \) the true label, and \( f^\theta \) a neural network with parameters \( \theta \). The loss function for adversarial training at epoch \( t \) is given by:

\begin{equation}
    \begin{aligned}[t]
    \mathcal{L}(f^\theta(X, y)) &= \mathcal{L_\text{0}}(f^\theta(X), y) + \lambda_t * \mathcal{L_\text{1}}(f^\theta(X^{\text{adv}}), y)
    \\& \text{where} \lambda_t \in [0,1]\ \text{and varies with epoch}\ t.
    \end{aligned}
\end{equation}

\( \mathcal{L}_\text{0} \) represents the standard cross-entropy loss computed on the clean input \( X \), and \( \mathcal{L}_\text{1} \) represents the adversarial loss computed on the perturbed input \( X^{\text{adv}} \). The weighting parameter \( \lambda_t \) is a function of the training epoch \( t \) and is scheduled to gradually increase during training to allow for smoother convergence (e.g  training begins with \( \lambda_t = 0 \) to focus on clean accuracy and then shift toward robustness by increasing \( \lambda_t \) to 1).

The adversarial example \( X^{\text{adv}} \) is generated using PGD, where the perturbation is constrained within an \( \ell_p \)-norm ball around the original input \( X \). Adjusting \( \lambda_t \) balances the trade-off between standard generalization and robustness.

\subsubsection{Interval Bound Propagation}
This method propagates the input bounds through every layer of the network, emulating the input with adversarial perturbation to compute certified bounds for the network. For an input $x \in \mathbb{R}^d$ applied to the neural network $f_{\theta}$ is perturbed by $\epsilon$ within the $\ell_p$ norm. Thus, the perturbation produces an interval $[x^l,\ x^u]\ =\ [x\ +\ \epsilon,\ x\ -\ \epsilon ]$. 

These intervals are further propagated across the network layers. As proposed in~\cite{gowal2018ibp}, we bound the activation of the intermediate layers within a box along the axis - $\ubar{z}_l(\epsilon) \leq z_l \leq \bar{z}_l(\epsilon)$.Thus, the bounds of the intermediate layers with $\ell_p$ norm adversarial perturbations  can be expressed as the following:
\begin{equation}
    z_{l,i}(\epsilon) = \min_{\substack{\ubar{z}_{l-1}(\epsilon) \leq z_{l-1} \leq \bar{z}_{l-1}(\epsilon)}} e_i^T h_k(z_{l-1})
\end{equation}
where $z_{l,i}(\epsilon)$ is the perturbed activation of the $i$-th coordinate of $z_l$ under the adversarial perturbation of size $\epsilon$, $\ubar{z}_{l-1}(\epsilon)$ and $\bar{z}_{l-1}(\epsilon)$  are the perturbed activation of the $l-1$ layer's lower and upper bound respectively, $h_l(\cdot)$ is the transformation of the layer $l$, $e_i^T$ is the transpose of the basis vector $e_i$ that selects the $i$-th component of the vector.
\\
For layers with linear transformation (affine layer / convolution layer) $l$, the pre-activation with $\ell_p$ norm perturbation is given by $h_k(z_{l-1}) = W_{l-1} \cdot z_{l-1} + b_l$. This can be computed by solving the optimization problem, efficiently represented with 2 matrix multiplications as in~\cite{gowal2018ibp}:
\begin{equation}
        \begin{aligned}[c]
            \mu_{l-1} &= \frac{\bar{z}_{l-1} + \ubar{z}_{l-1}}{2} \\ 
            r_{l-1} &= \frac{\bar{z}_{l-1} - \ubar{z}_{l-1}}{2} \\ 
            \mu_{l} &=  W_l \cdot \mu_{l-1} + b_l \\ 
            r_l &= |W| \cdot r_{l-1} \\ 
            \ubar{z}_l &= \mu_{l} + r_l \\ 
            \bar{z}_l &= \mu{l} - r_l 
        \end{aligned}
\end{equation}
where $\mu_{l-1}$ is the center of the interval, $r_{l-1}$ is the radius of the interval $l-1$, \( W_l \), \( b_l \) are the weight matrix and bias vector respectively. \\
For a monotonic layer after the application of the activation function $\sigma$ is defined by:
\begin{equation}
    \begin{aligned}[c]
        \ubar{z}_{l} &= h_l(\ubar{z}_{l-1}) \\
        \bar{z}_{l} &= h_l(\bar{z}_{l-1})
    \end{aligned}
\end{equation}
The key strength of IBP lies in its computational efficiency. However, it is worth to note that this comes with the drawback of loose approximations, especially for networks with large inputs. The propagated bounds tend to be loose, leading to broad estimates of robustness bounds.

\subsubsection{Spatio-Temporal Systems}
Spatio-temporal systems refer to neural network architectures that process inputs exhibiting inherent correlations along both spatial and temporal dimensions. Formally, such systems implement mappings $f : \mathbb{R}^{C \times D \times H \times W} \rightarrow \mathcal{Y}$, where $C$ denotes the number of input channels, $D$ the temporal depth (e.g., number of frames or slices), and $H \times W$ the spatial resolution. Prominent examples include video classification networks processing sequential frames captured over time, and volumetric medical imaging models analyzing sequential slices from magnetic resonance imaging (MRI) or computed tomography (CT) scans. A distinguishing characteristic of spatio-temporal data is that realistic perturbations - whether arising from sensor noise, environmental artifacts, or adversarial manipulation - are not independent across the temporal axis. For instance, an adversarial patch affixed to a physical object (e.g., a manipulated road sign or bumper sticker) induces perturbations that are spatially localized and temporally correlated across consecutive frames. This structure reduces the effective adversarial degrees of freedom compared to unconstrained per-pixel perturbations, a property that can be exploited to obtain tighter verification bounds than those achievable with standard interval bound propagation methods that treat all input dimensions as independent.

\subsection{Challenge with Standard IBP in Video Settings}
Standard IBP becomes inefficient when applied to inputs with high correlation, such as videos. In video classification, the input is a sequence of  $N$ frames. Due to high correlation along the temporal axis, the independent propagation of bounds for each frame leads to an \textit{exponential} or \textit{polynomial} increase in the size of the propagated bounds.

For example, let $N$ be the number of frames in a video, and suppose the perturbation is applied independently to each frame. The propagated bounds across all frames grow as $O(N^d)$, where $d$ is the input dimension. This results in bounds that are overly conservative and difficult to scale for large networks. To address this, we can introduce shared constraints across the temporal axis, allowing for more efficient certification in video models.

\subsection{Proof of the Closed-Form first-layer bounds}
\label{app:proof-closedform}
\setcounter{theorem}{0}
\begin{theorem}[Closed-Form First-Layer Bounds]
	For each neuron \(j\) in the first layer, the exact bounds are given by
	\[
		\ell_j^{(1)} = z^{(1)}_{\mathrm{clean},j} - \sum_{r=1}^{q} \epsilon_r |\alpha_{j,r}|, \quad
		u_j^{(1)} = z^{(1)}_{\mathrm{clean},j} + \sum_{r=1}^{q} \epsilon_r |\alpha_{j,r}|.
	\]
\end{theorem}
\begin{proof}
	The objective is linear and the feasible set is a hyperrectangle. By linear programming theory, extrema are attained at vertices where each \(\tilde{\delta}_r \in \{\pm \epsilon_r\}\). Choosing signs aligned with \(\alpha_{j,r}\) maximizes the objective, yielding the stated expressions. 
\end{proof}

\subsection{Algorithm for Lipschitz Propagation for subsequent layers of the network}
Algorithm~\ref{alg:lipschitz} defines Lipschitz bound propagation for the remaining layers of the network, starting from the second layer.

\begin{algorithm}[H]
    \caption{Lipschitz Bound Propagation}
    \label{alg:lipschitz}
    \begin{algorithmic}[1]
        \REQUIRE Bounds $[\ell^{(1)}, u^{(1)}]$, network layers
        \STATE Initialize $L_{\mathrm{total}} \gets 1$
        \FOR{$i = 2$ to $k$}
        \STATE Compute spectral norm $L_i$
        \STATE Update $L_{\mathrm{total}} \gets L_{\mathrm{total}} \cdot L_i$
        \STATE Propagate centre forward
        \ENDFOR
        \RETURN Output bounds using $L_{\mathrm{total}}$
    \end{algorithmic}
\end{algorithm}

\subsection{Proof of the soundness of Lipschitz Propagation}
\label{app:proof-lipschitz}

\begin{theorem}[Soundness of Lipschitz Bound Propagation after Closed-Form Analysis]
Let $f = f^{(k)} \circ \cdots \circ f^{(1)}$ be a neural network, where the first
layer $f^{(1)}$ admits an exact closed-form bound computation under an
$\ell_\infty$-bounded perturbation set
\[
T(\mathbf{x}) = \{ \mathbf{x}' : \|\mathbf{x}' - \mathbf{x}\|_\infty \le \epsilon \}.
\]
Suppose the remaining subnetwork
\(
\tilde{f} := f^{(k)} \circ \cdots \circ f^{(2)}
\)
is $L_{\mathrm{total}}$-Lipschitz with respect to the $\ell_2$-norm.

Let $[\ell^{(1)}, u^{(1)}]$ denote the exact first-layer bounds obtained via the
closed-form analysis, and let
\[
\mathbf{c}^{(1)} = \tfrac{1}{2}(\ell^{(1)} + u^{(1)}), \qquad
\mathbf{r}^{(1)} = \tfrac{1}{2}(u^{(1)} - \ell^{(1)}).
\]
Then, for all $\mathbf{x}' \in T(\mathbf{x})$, the network output satisfies
\[
f(\mathbf{x}') \in \left[ f(\mathbf{c}^{(1)}) \pm L_{\mathrm{total}}\sqrt{d}\,\|\mathbf{r}^{(1)}\|_\infty \mathbf{1} \right]
\]
where $d$ denotes the dimensionality of the input to $\tilde{f}$ and
$\mathbf{1}$ is the all-ones vector.
\end{theorem}
\begin{proof}
By construction of the closed-form solution for the first layer, for every
$\mathbf{x}' \in T(\mathbf{x})$ the corresponding first-layer activation
$\mathbf{z}^{(1)}(\mathbf{x}')$ satisfies
\[
\mathbf{z}^{(1)}(\mathbf{x}') \in [\ell^{(1)}, u^{(1)}].
\]
Equivalently, this can be written in center--radius form as
\[
\mathbf{z}^{(1)}(\mathbf{x}') = \mathbf{c}^{(1)} + \boldsymbol{\eta},
\quad
\text{with }
\|\boldsymbol{\eta}\|_\infty \le \|\mathbf{r}^{(1)}\|_\infty.
\]

To apply Lipschitz-based analysis, we convert this enclosure to the $\ell_2$-norm.
By standard norm equivalence in finite-dimensional spaces, for any vector
$\boldsymbol{\eta} \in \mathbb{R}^d$,
\[
\|\boldsymbol{\eta}\|_2 \le \sqrt{d} \, \|\boldsymbol{\eta}\|_\infty.
\]
Therefore,
\[
\|\mathbf{z}^{(1)}(\mathbf{x}') - \mathbf{c}^{(1)}\|_2
\le
\sqrt{d} \, \|\mathbf{r}^{(1)}\|_\infty.
\]

Since $\tilde{f}$ is $L_{\mathrm{total}}$-Lipschitz with respect to the $\ell_2$-norm,
it follows that
\[
\|\tilde{f}(\mathbf{z}^{(1)}(\mathbf{x}')) - \tilde{f}(\mathbf{c}^{(1)})\|_2
\le
L_{\mathrm{total}} \sqrt{d} \, \|\mathbf{r}^{(1)}\|_\infty.
\]

Finally, using the fact that the $\ell_\infty$-norm is upper bounded by the
$\ell_2$-norm, we obtain
\[
\|\tilde{f}(\mathbf{z}^{(1)}(\mathbf{x}')) - \tilde{f}(\mathbf{c}^{(1)})\|_\infty
\le
L_{\mathrm{total}} \sqrt{d} \, \|\mathbf{r}^{(1)}\|_\infty.
\]

Substituting $\tilde{f} \circ f^{(1)} = f$ yields the stated output bounds. Thus,
the Lipschitz bound propagation applied after the closed-form first-layer analysis,
together with norm conversion, produces a sound enclosure of the network output
under the original $\ell_\infty$-bounded perturbation specification.
\end{proof}

\subsection{Proof of the soundness of Löwner-John Ellipsoid Sampling}
\label{app:proof-ljsampling}
\begin{theorem}[Soundness of Löwner--John Ellipsoid Propagation]
Let $f = f^{(k)} \circ \cdots \circ f^{(2)}$ denote the subnetwork following the
first layer, and let $[\ell^{(1)}, u^{(1)}]$ be the exact first-layer bounds
obtained via the closed-form analysis under an $\ell_\infty$-bounded perturbation
set $T(\mathbf{x})$. Let $\mathcal{Z}^{(1)} \subset \mathbb{R}^d$ denote the exact
reachable set of first-layer activations, and let $\mathcal{E}$ be the
Löwner--John ellipsoid of $\mathcal{Z}^{(1)}$, i.e., the minimum-volume ellipsoid
such that
\[
\mathcal{Z}^{(1)} \subseteq \mathcal{E}.
\]

Assume that the bounds on the network output are computed by evaluating $f$ on a
set of points whose convex hull contains $f(\partial \mathcal{E})$, where
$\partial \mathcal{E}$ denotes the boundary of $\mathcal{E}$. Then the resulting
output bounds constitute a sound enclosure of
\[
\{ f(\mathbf{z}) : \mathbf{z} \in \mathcal{Z}^{(1)} \}.
\]
\end{theorem}

\begin{proof}
By construction of the closed-form first-layer solution, the exact reachable set
$\mathcal{Z}^{(1)}$ of first-layer activations is a convex polytope defined as the
image of a hyperrectangle under a linear map. By definition of the Löwner--John
ellipsoid, $\mathcal{E}$ is the unique minimum-volume ellipsoid satisfying
\[
\mathcal{Z}^{(1)} \subseteq \mathcal{E}.
\]

Since $f$ is a composition of affine mappings and ReLU nonlinearities, it is
continuous on $\mathbb{R}^d$. Therefore, the image $f(\mathcal{E})$ is a compact
set, and the extreme values of each output coordinate are attained on the
boundary $\partial \mathcal{E}$.

Let $g_j(\mathbf{z})$ denote the $j$-th output coordinate of $f$. Because
$g_j$ is continuous and $\mathcal{E}$ is compact, the extrema
\[
\max_{\mathbf{z} \in \mathcal{E}} g_j(\mathbf{z}), \qquad
\min_{\mathbf{z} \in \mathcal{E}} g_j(\mathbf{z})
\]
are attained at points in $\partial \mathcal{E}$.

Suppose that the output bounds are constructed from a set of evaluated points
$\{\mathbf{z}_i\} \subset \partial \mathcal{E}$ whose image under $f$ has a convex
hull enclosing $f(\partial \mathcal{E})$. Then, for every $\mathbf{z} \in
\mathcal{E}$ and every output coordinate $j$,
\[
g_j(\mathbf{z}) \in
\left[
\min_i g_j(\mathbf{z}_i), \;
\max_i g_j(\mathbf{z}_i)
\right].
\]

Since $\mathcal{Z}^{(1)} \subseteq \mathcal{E}$, it follows immediately that
\[
f(\mathcal{Z}^{(1)}) \subseteq
\left[
\min_i f(\mathbf{z}_i), \;
\max_i f(\mathbf{z}_i)
\right],
\]
which establishes soundness of the resulting bounds with respect to all
admissible perturbations in $T(\mathbf{x})$.

In practice, when only a finite subset of boundary points of $\mathcal{E}$ is
evaluated, the resulting enclosure constitutes a conservative over-approximation
of $f(\mathcal{E})$. Increasing the sampling density monotonically tightens the
bounds while preserving soundness. \qedhere
\end{proof}

\subsection{Proof that IBP results in upper bounds than the closed form solution}
\label{app:proof-ibp-greater-than-closedform}
\begin{lemma}
    Let \( \mathbf{z} = \mathbf{W}(\mathbf{x} + \boldsymbol{\delta}) + \mathbf{b} \in \mathbb{R}^m \) denote the pre-activation of a linear layer applied to a perturbed input \( \mathbf{x} + \boldsymbol{\delta} \in \mathbb{R}^n \), where \( \boldsymbol{\delta} \in T(\mathbf{x}) \) satisfies the input specification. Suppose that the input specification includes a shared constraint \( \delta_i = \delta_j \), and that \( w_{k,i} \cdot w_{k,j} < 0 \) for some output coordinate \( k \). Then the upper bound on \( z_k \) computed via IBP is strictly greater than the upper bound computed by the closed form solution: $\max_{\boldsymbol{\delta} \in T(\mathbf{x})} z_k^\text{closed-form} \; < \; \max_{\boldsymbol{\delta} \in \text{Int. Const.}} z_k^\text{IBP}.$
    This tightness lemma generalizes straight-forwardly to the lower-bound as well.
\end{lemma}

\begin{proof}
    Interval bound propagation assumes independent bounds for each \( \delta_i \), i.e., it allows \( \delta_i \in [-\epsilon_i, \epsilon_i] \) and \( \delta_j \in [-\epsilon_j, \epsilon_j] \) independently, so the worst-case affine response includes combinations such as \( \delta_i = \epsilon_i \), \( \delta_j = -\epsilon_j \). When \( w_{k,i} > 0 \) and \( w_{k,j} < 0 \), such opposite-signed perturbations drive both terms in \( z_k \) to increase simultaneously: $w_{k,i} \delta_i + w_{k,j} \delta_j \to w_{k,i} \epsilon_i + w_{k,j} (-\epsilon_j).$
    However, under the constraint \( \delta_i = \delta_j = \delta \), this extremal combination is infeasible. The closed form solution correctly optimizes over the reduced feasible region defined by \( \delta_i = \delta_j \) for the first layer, and can only take values of \( \delta \) that affect both coordinates simultaneously. In this setting, the optimal value of
    $w_{k,i} \delta + w_{k,j} \delta = (w_{k,i} + w_{k,j}) \delta$
    is strictly smaller in magnitude than what IBP would compute, since the weights have opposing signs and partially cancel. Hence, the closed form solution yields a tighter (lower) upper bound in the first layer.
\end{proof}

\subsection{Additional Experimental Results and Runtime Details}
\label{app:resultsandruntimes}
The table~\ref{tab:verifier-runtimes} summarises the average time taken per sample by the various verifiers categorized into the types of models. The STBP-Patch verifiers consume the highest amount of time per sample owing to the additional computations involved in constraining perturbations to patches.

Detailed results of the experiments evaluating the different variants of Spatio-Temporal Bound Propagation (STBP) are presented in Table~\ref{tab:detailed-results}. This table provides a comprehensive comparison of the certified robustness achieved by the standard STBP method without adversarial patch constraints, as well as the results obtained using the patch-based STBP variant. In particular, the patch-based variant demonstrates improved certified bounds across multiple spatio-temporal tasks.

\begin{table}
\begin{center}
\resizebox{\linewidth}{!}{%
    \begin{tabular}{@{}lll@{}}
        \toprule
        Experiment & Verifier & Runtime per Sample (s) \\ \midrule
        MNIST Toy Large & IBP & 0.00482 \\
        MNIST Toy Large & CROWN-IBP & 0.0358 \\
        MNIST Toy Large & STBP-Lipschitz & 0.02337 \\
        MNIST Toy Large & STBP-Lipschitz-Patch & 0.0171 \\
        Udacity Steering Angle & IBP & 0.128 \\
        Udacity Steering Angle & CROWN-IBP & 0.4870881 \\
        Udacity Steering Angle & STBP-Lipschitz & 0.20703473\\
        Udacity Steering Angle & STBP-Lipschitz-Patch & 0.18243989 \\
        MEDMNIST Synapse3D (32 X 32 X 32) & IBP & 0.00268376 \\
        MEDMNIST Synapse3D (32 X 32 X 32) & CROWN-IBP & 0.00781921 \\
        MEDMNIST Synapse3D (32 X 32 X 32) & STBP-Lipschitz & 0.00855367 \\
        MEDMNIST Synapse3D (32 X 32 X 32) & STBP-Lipschitz-Patch & 0.00575141 \\
        MEDMNIST Synapse3D (64 X 64 X 64) & IBP & 0.00743503 \\
        MEDMNIST Synapse3D (64 X 64 X 64) & CROWN-IBP & 0.02067797 \\
        MEDMNIST Synapse3D (64 X 64 X 64) & STBP-Lipschitz & 0.01953672 \\
        MEDMNIST Synapse3D (64 X 64 X 64) & STBP-Lipschitz-Patch & 0.01275141 \\
        UCF-101 & IBP & 0.072 \\
        UCF-101 & CROWN-IBP &  0.272 \\
        UCF-101 & STBP-Lipschitz & 0.488 \\
        UCF-101 & STBP-Lipschitz-Patch & 0.332 \\ 
        \bottomrule
    \end{tabular}%
}
\end{center}
\caption{Runtime per sample for IBP, STBP-Lipschitz, STBP-Lipschitz Patch Verification}
\label{tab:verifier-runtimes}
\end{table}

\begin{sidewaystable*}[!htbp]
\begin{center}
\resizebox{\textwidth}{!}{%
\begin{tabular}{@{}lccccccccccccccccccccccccl@{}}
\toprule
\multirow{2}{*}{Experiment} & \multirow{2}{*}{Clean Accuracy} & \multicolumn{3}{c}{Model Parameters} & \multicolumn{5}{c}{IBP Rob. Acc. (\%)} & \multicolumn{5}{c}{CROWN-IBP Rob. Acc. (\%)} & \multicolumn{5}{c}{STBP-Lipschitz Adv. Rob. Acc. (\%)} & \multicolumn{5}{c}{STBP-Lipschitz Adv. Patch Rob. Acc. (\%)} \\ 
\cmidrule(l){3-5} \cmidrule(l){6-10} \cmidrule(l){11-15} \cmidrule(l){16-20} \cmidrule(l){21-25}
 &  & \multicolumn{1}{l}{Input dim} & \multicolumn{1}{l}{Output dim} & \multicolumn{1}{l}{No. of Samples} & \multicolumn{1}{l}{$\epsilon=0.1$} & \multicolumn{1}{l}{$\epsilon=0.01$} & \multicolumn{1}{l}{$\epsilon=0.001$} & \multicolumn{1}{l}{$\epsilon=0.0001$} & \multicolumn{1}{l}{$\epsilon=0.00001$} & \multicolumn{1}{l}{$\epsilon=0.1$} & \multicolumn{1}{l}{$\epsilon=0.01$} & \multicolumn{1}{l}{$\epsilon=0.001$} & \multicolumn{1}{l}{$\epsilon=0.0001$} & \multicolumn{1}{l}{$\epsilon=0.00001$} & \multicolumn{1}{l}{$\epsilon=0.1$} & \multicolumn{1}{l}{$\epsilon=0.01$} & \multicolumn{1}{l}{$\epsilon=0.001$} & \multicolumn{1}{l}{$\epsilon=0.0001$} & \multicolumn{1}{l}{$\epsilon=0.00001$} & \multicolumn{1}{l}{$\epsilon=0.1$} & \multicolumn{1}{l}{$\epsilon=0.01$} & \multicolumn{1}{l}{$\epsilon=0.001$} & \multicolumn{1}{l}{$\epsilon=0.0001$} & \multicolumn{1}{l}{$\epsilon=0.00001$} \\
 \midrule
MNIST Toy Model & 93.1\% & 1 x 5 x 8 x 8 & 10 & 1000 
& 0.0 & 0.0 & 1.2 & 53.2 & 89.1 
& 0.0 & 17.9 & 69.96 & 82.0 & 82.25 
& 0.0 & 15.15 & 63.0 & 91.2 & \textbf{93.10} 
& \textbf{0.4} & \textbf{20.45} & \textbf{79.5} & \textbf{92} & \textbf{93.10} \\
MNIST Toy Model Large & 91.3\% & 1 x 10 x 28 x 28 & 10 & 1000 
& 0.0 & 0.0 & 53.2  & 69.1 & 71.0 
& 0.0 & 15.8 & 87.4 & 88.2 & 88.3 
& 0.0 & 83.2 & 90.60 & \textbf{91.3} & \textbf{91.3} 
& \textbf{63.7} & \textbf{89.2} & \textbf{91.1} & \textbf{91.2} & \textbf{91.3} \\
UCF-101 & 74.41\% & 3 x 30 x 32 x 32 & 5 & 500 
& 0.0 & 0.0 & 0.0 & 20.03 & 26.36 
& 0.0 & 0.0 & 2.07 & 49.93 & 63.48 
& 0.0 & 7.3 & 18.36 & 48.6 & 67.3 
& \textbf{2.6} & \textbf{12.3} & \textbf{19.36} & \textbf{61.9} & \textbf{71.81} \\
Udacity Steering Angle & 84.23\% & 3 x 30 x 32 x 32 & 3 & 1123 
& 0.0 & 0.0 & 0.0 & 70.0 & 75.0 
& 0.0 & 0.0 & 5.61 & \textbf{84.24} & \textbf{84.24} 
& 0.0 & 0.0 & 74.26 & \textbf{84.24}& \textbf{84.24}
& \textbf{7.24} & \textbf{84.24} & \textbf{84.24} & \textbf{84.24} & \textbf{84.24} \\
MEDMNIST Synapse3D & 72.88\% & 32 x 32 x 32 & 2 & 177 
& 0.0 & 0.0 & 19.77 & 72.88 & 72.88 
& 0.0 & 0.0 & 71.75 & \textbf{72.88} & \textbf{72.88}
& 0.0 & 1.12 & \textbf{72.88} & \textbf{72.88} & \textbf{72.88} 
& 0.0 & \textbf{72.88} & \textbf{72.88} & \textbf{72.88} & \textbf{72.88} \\ 
MEDMNIST Synapse3D & 73.01\% & 64 x 64 x 64 & 2 & 352 
& 0.0 & 0.0 & 0.0 & 73.01 & 73.01 
& 0.0 & 0.0 & 8.24 & \textbf{73.01} & \textbf{73.01} 
& 0.0 & 0.0 & 0.0 & \textbf{73.01} & \textbf{73.01} 
& 0.0 & 0.0 & \textbf{33.33} & \textbf{73.01} & \textbf{73.01} \\
\bottomrule
\end{tabular}%
}
\end{center}
\caption{Certified robustness of spatio-temporal models using IBP, STBP and STBP under adversarial patch perturbations}
\label{tab:detailed-results}
\end{sidewaystable*}